\documentclass[twoside,11pt]{article}

\def\bSig\mathbf{\Sigma}

\usepackage{macrosArticle}
\usepackage[normalem]{ulem}
\usepackage{xcolor}
\usepackage{booktabs}
\usepackage{longtable}
\usepackage{makecell}
\usepackage{arydshln}
\usepackage{natbib}

\newcommand{\Opt}{\mathrm{MED}}

\usepackage[figuresright]{rotating}

\newcommand{\eff}{\text{eff}}

\newcommand{\Set}[1]{\mathchoice%
{\left\{ #1 \right\}}{\{ #1 \}}{\{ #1 \}}{\{ #1 \}}}
\newcommand{\E}{\mathds{E}}
\renewcommand{\P}{\mathbb{P}}
\newcommand{\kl}{\mathrm{kl}}
\newcommand{\ymid}{y_{\mathrm{mid}}}

\newcommand{\tblopt}[1]{\underline{#1}} 
\newcommand{\tblwinrec}[1]{\textbf{#1}} 

\usepackage{tikz}

\newcommand{\dash}[1]{%
    \tikz[baseline=(todotted.base)]{
        \node[inner sep=1pt,outer sep=0pt] (todotted) {#1};
        \draw[dashed] ([yshift=-2pt]todotted.south west) -- ([yshift=-2pt]todotted.south east);
    }%
}%

\title{On Multi-Armed Bandit Designs for Dose-Finding Trials}

\author{Maryam Aziz$^{1}$, Emilie Kaufmann$^{2}$ and Marie-Karelle Riviere$^{3}$ \\ \\
\small $^1$Spotify, $^2$CNRS \& Univ. Lille,  CRIStAL (UMR 9189), Inria SequeL \\ \small $^3$ Statistical Methodology Group, Biostatistics and Programming department, Sanofi R\&D} 

\date{}

\begin{document}
 	   
\maketitle

\begin{abstract}
We study the problem of finding the optimal dosage in early stage clinical trials through the multi-armed bandit lens. We advocate the use of the Thompson Sampling principle, a flexible algorithm that can accommodate different types of monotonicity assumptions on the toxicity and efficacy of the doses. For the simplest version of Thompson Sampling, based on a uniform prior distribution for each dose, we provide finite-time upper bounds on the number of sub-optimal dose selections, which is unprecedented for dose-finding algorithms. Through a large simulation study, we then show that variants of Thompson Sampling based on more sophisticated prior distributions outperform state-of-the-art dose identification algorithms in different types of dose-finding studies that occur in phase I or phase I/II trials.
\end{abstract}


\maketitle

\section{Introduction}

Multi-armed bandit models were originally introduced in the 1930's as a simple model for a (phase III) clinical trial in which one control treatment is tried against one alternative \citep{Thompson33}. While those models are nowadays widely studied with completely different applications in mind, like online advertisement \citep{LiChapelle11}, recommender systems \citep{LiCLS10News} or cognitive radios \citep{Anandkumar11}, there has been a surge of interest in the use of bandit algorithms for clinical trials (see \cite{Villar15BCT}). More broadly, Adaptive Clinical Trials have received an increased attention \citep{GuidelinesACT18} as the Food and Drug Administration recently updated a draft of guidelines for their actual use \citep{FDA19}. In this paper, we focus on adaptive designs for phase I and phase I/II clinical trials for single-agent in oncology, for which adaptations of the original bandit algorithms may be of interest.  

Phase I trials are the first stage of testing in human subjects.  Their goal is to evaluate the safety (and feasibility) of the treatment and identify its side effects. For non-life-threatening diseases, phase I trials are usually conducted on human volunteers. In life-threatening diseases such as cancer or AIDS, phase I studies are conducted with patients because of the aggressiveness and possible harmfulness of the treatments, possible systemic treatment effects, and the high interest in the new drug's efficacy in those patients directly. The aim of a phase I dose-finding study is to \emph{determine the most appropriate dose level that should be used in further phases of the clinical trials}. Traditionally, the focus is on determining the highest dose with acceptable toxicity called the Maximum Tolerated Dose (MTD). Once the initial safety of the drug has been confirmed in phase I trials, phase II trials are performed on larger groups and are designed to establish the efficacy of the drug and confirm the safety identified in phase I. In phase II dose-finding studies, the dose-efficacy relationship is modeled in order to estimate the smallest dose to obtain a desired efficacy, called the minimal effective dose (MED). Approaches that use both efficacy and toxicity to find an optimal dose are called phase I/II designs. If the new potential treatment shows some efficacy in phase II, it is compared to alternative treatments in phase III. We here consider two classes of algorithms for dose-finding in early stage trials:
algorithms which consider only toxicity, suited for phase I trials,
and algorithms which consider both toxicity and efficacy, suited for phase I/II trials.

Until recently, cytotoxic agents were the main agent of anti-tumor drug development. A common assumption for these agents is that both toxicity and efficacy of the treatment are monotonically increasing with the dose \citep{chevret06}. Hence, only toxicity is required to determine the optimal dose which is then the Maximum Tolerated Dose. From a statistical perspective, the MTD is often defined as the dose level closest to an acceptable targeted toxicity probability fixed prior to the trial onset \citep{faries94,storer89}. However, Molecularly Targeted Agents (MTAs) have emerged as a new treatment option in oncology that have changed the practice of cancer patient care \citep{Postel-Vinay09,letourneau10,letourneau11,letourneau12}. Previously-common assumptions do not necessarily hold for MTAs. Although toxicity is still assumed to be increasing with the dose, it may be so low that the trial cannot be driven by toxicity occurrence only. Efficacy needs to be studied jointly with toxicity, so that the most appropriate dose is not just the MTD. In particular, for some mechanisms of action, a plateau of efficacy can be observed when increasing the dose \citep{hoering11}, for instance when the targeted receptors are saturated. In this paper, we aim at providing a unified approach that can be used both for phase I trials involving cytotoxic agents and phase I/II trials involving MTAs.

Phase I cytotoxic clinical trials in oncology involve several ethical concerns. Therefore, in order to gather information about the dose-toxicity relationship it is not possible to include a large number of patients and randomize them at each different dose level considered in the trial. Patients treated with dose levels over the MTD would be exposed to very high toxicity, and patients treated at low dose levels would be administrated ineffective dose levels. In addition, the total sample size is often very limited. For these reasons, the doses to be allocated should be selected sequentially, taking into account the outcomes of the previous allocated doses, with ideally two objectives in mind: finding the MTD (which is crucial for the next stages of the trial) and treating as many trial participants as possible with this MTD. This trade-off between treatment (curing patients during the study) and experimentation (finding the best treatment) is a common issue in clinical trials. By viewing optimal dose identification as a particular multi-armed bandit problem, this trade-off can be rephrased as a trade-off between  rewards and error probability, two performance measures that are well-studied in the bandit literature and that are known to be somewhat antagonistic (see \cite{Bubeckal11,ESAIM17}). 

In this paper, we investigate the use of Thompson Sampling \citep{Thompson33} for dose-finding clinical trials. This Bayesian algorithm has gained a lot of popularity in the machine learning community for its successful use for reward maximization in bandit models (see, e.g., \cite{LiChapelle11}). Interestingly, in the growing literature on Bayesian Adaptive Designs \citep{Berry06BAD,Berry10BAD}, several designs that may be viewed as variants of Thompson Sampling have been proposed for other types of clinical trials in which different treatments are compared \citep{Thall07,Berrys16Alzheimer}. However, to the best of our knowledge, the use of Thompson Sampling has not been investigated yet for dose-finding trials, and the present paper aims to fill this gap. We show that, unlike other bandit algorithms that are better suited for phase III trials, Thompson Sampling can indeed be naturally adapted to dose-finding trials. 

Our first contribution is a theoretical study in the context of MTD identification showing that the simplest version of Thompson Sampling based on independent prior distributions for each arm asymptotically minimizes the number of sub-optimal allocations during the trial. Albeit asymptotic, this sanity-check for Thompson Sampling with a simple prior motivates our investigation for its use with more realistic prior distributions, where theoretical guarantees are harder to obtain. Our second contribution is to show that Thompson Sampling using more sophisticated prior distributions can compete with state-of-the art dose-finding algorithms. We indeed show that the algorithm can exploit the monotonicity assumption on the toxicity probabilities that are common for MTD identification (Section~\ref{sec:TSIncreasing}), but also deal with more complex assumptions on both the toxicity and efficacy probabilities that are relevant for trials involving MTAs (Section~\ref{sec:TSEff}). Through extensive experiments on simulated clinical trials we show that our Thompson Sampling variants typically outperform state-of-the-art dose-finding algorithms. Finally, we propose a discussion revisiting the treatment versus experimentation trade-off through a bandit lens, and explain why an adaptation of existing best arm identification designs \citep{Bubeck10BestArm,Karnin13} seems currently less promising for dose-finding clinical trials.

The paper is structured as follows. In Section~\ref{sec:Bandits}, we present a multi-armed bandit (MAB) model for the MTD identification problem and introduce the Thompson Sampling algorithm. In Section~\ref{sec:Analysis}, we propose an analysis of Thompson Sampling with independent Beta priors on the toxicity of each dose: We provide finite-time upper-bounds on the number of sub-optimal selections, which match an (asymptotic) lower bound on those quantities. Then in Section \ref{sec:TS}, we show that Thompson Sampling can leverage the usual monotonicity assumptions in dose-finding clinical trials. 
In Section~\ref{sec:Experiments}, we report the results of a large simulation study to assess the quality of the proposed design. Finally in Section~\ref{sec:Discussion}, we propose a discussion on the use of alternative bandit methods. 
                  
\section{Maximum Tolerated Dose Identification as a Bandit Problem}\label{sec:Bandits}

In this section, we propose a simple statistical model for the MTD identification problem in phase I clinical trials and show that it can be viewed as a particular multi-armed bandit problem.

A dose-finding study involves a number $K$ of dose levels that have been chosen by physicians based on preliminary experiments ($K$ is usually a number between $3$ and $10$). Denoting by $p_k$ the (unknown) toxicity probability of dose $k$, the Maximum Tolerated Dose (MTD) is defined as the dose with a toxicity probability closest to a target:
\[k^* \in \argmin{k \in \{1,\dots,K\}} |\theta - p_k|,\]
where $\theta$ is the pre-specified targeted toxicity probability (typically between 0.2 and 0.35). For clinical trials in life-threatening diseases, efficacy is often assumed to be increasing with toxicity, hence the MTD is the most appropriate dose to further investigate in the rest of the trial. However, we shall see in Section~\ref{sec:TS} that under different assumptions the optimal dose may be defined differently. 

\subsection{A (Bandit) Model for MTD Identification}

A MTD identification algorithm proceeds sequentially: at round $t$ a dose $D_t \in \{1,\dots,K\}$ is selected and administered to a patient for whom a toxicity response is observed. A binary outcome $X_t$ is revealed where $X_t = 1$ indicates that a harmful side-effect occurred and $X_t=0$ indicates than no harmful side-effect occurred. We assume that $X_t$ is drawn from a Bernoulli distribution with mean $p_{D_t}$ and is independent from previous observations. The \emph{selection rule} for choosing the next dose level to be administered is sequential in that it uses the past toxicity observations to determine the dose to administer to the next patient. More formally, $D_t$ is $\cF_{t-1}$-measurable where $\cF_t= \sigma(U_0,D_1,X_1,U_1,\dots,D_t,X_t,U_t)$ is the $\sigma$-field generated by the observations made with the first $t$ patients and the possible exogenous randomness used in each round $t$, $U_{t-1} \sim \cU([0,1])$. Along with this selection rule, a ($\cF_{t}$-measurable) \emph{recommendation rule} $\hat{k}_t$ indicates which dose would be recommended as the MTD, if the experiments were to be stopped after $t$ patients. 

Usually in clinical trials the total number of patients $n$ is fixed in advance and the first objective is to ensure that the dose $\hat{k}_n$ recommended at the end of the trial is close to the MTD, $k^*$, but there is also an incentive to treat as many patients as possible with the MTD during the trial. Letting $N_k(t) = \sum_{s=1}^t\ind_{(D_s = k)}$ be the number of time dose $k$ has been given to one of the first $t$ patients, this second objective can be formalized as that of minimizing $N_k(n)$ for $k\neq k^*$.  In the clinical trial literature,  empirical evaluations of dose-finding designs usually report both the empirical distribution of the recommendation strategy $\hat{k}_n$ (that should be concentrated on the MTD) and estimates of $\bE[N_k(n)]/n$ for all doses $k$ to assess the quality of the selection strategy in terms of allocating MTD as often as possible.

The sequential interaction protocol described above is reminiscent of a stochastic multi-armed bandit (MAB) problem (see \cite{BanditBook18} for a recent survey). A MAB model refers to a situation in which an agent sequentially chooses arms (here doses) and gets to observe a realization of an underlying probability distribution (here a Bernoulli distribution with mean being the probability that the chosen dose is toxic). Different objectives have been considered in the bandit literature, but most of them are related to \emph{learning the arm with largest mean}, whereas in the context of clinical trials we are rather concerned with the arm which is the closest to some threshold. 

\subsection{Thompson Sampling for MTD Identification}

Early works on bandit models \citep{Robbins52,LaiRobbins85bandits} mostly consider a \emph{reward maximization} objective: The samples $(X_t)$ are viewed as rewards, and the goal is to maximize the sum of these rewards, which boils down to choosing the arm with largest mean as often as possible. This problem was originally introduced in the 1930s in the context of phase III clinical trials \citep{Thompson33}. In this context, each arm models the response to a particular treatment, and maximizing rewards amounts to giving the treatment with largest probability of success to as many patients as possible. This suggests a phase III trial is designed for treating as many patients as possible with the best treatment rather than identifying it. The trade-off between treatment and identification is also relevant for MTD identification: besides finding the MTD another objective is to treat as many patients as possible with it during the trial.

Reward maximization in a Bernoulli bandit model is a well-studied problem \citep{Jacko19Binary}. In particular, it is known since \citep{LaiRobbins85bandits} that any algorithm that performs well on every bandit instance should select each sub-optimal arm $k$ more than $C_k\log(n)$ times, where $C_k$ is some constant, in a regime of large values of $n$. Algorithms with finite-time upper bounds on the number of sub-optimal selections have been exhibited \citep{Aueral02,Audibertal09UCBV}, some of which match the aforementioned lower bound on the number of sub-optimal selections \citep{KLUCBJournal}. In the context of MTD identification, we are also concerned about \emph{ minimizing the number of sub-optimal selections} but with a different notion of optimal arm: the MTD instead of the arm with largest mean. 

Algorithms for maximizing rewards in a bandit model mostly fall in two categories: frequentist algorithms, based on upper-confidence bounds (UCB) for the unknown means of the arms (popularized by \cite{KatRob:95Gauss,Aueral02}) and Bayesian algorithms, that exploit a posterior distribution on the means (see, e.g. \cite{Powell12Book,AISTATS12}). Among those, Thompson Sampling (TS)  is a popular approach, known for its practical successes beyond simple bandit problems \citep{AGContext13,Agrawal17TSRL}. In the context of clinical trials,  variants of Thompson Sampling have been notably studied for phase III clinical trials involving two treatments (see \cite{Thall07} and references therein), or for adaptive trials involving interim analyses \citep{Berrys16Alzheimer}. Strong theoretical properties have also been established for this algorithm in simple models. In particular, Thompson Sampling was proved to be asymptotically optimal for Bernoulli bandit models \citep{ALT12,AGAISTAT13}.

Thompson Sampling, also known as probability matching, implements the following simple Bayesian heuristic. Given a prior distribution over the arms, at each round an arm is selected at random according to its posterior probability of being optimal. In this paper, we advocate the use of Thompson Sampling for dose-finding, using the appropriate notion of optimality. In particular, Thompson Sampling for MTD identification consists of selecting a dose at random according to its posterior probability of being the MTD. Given a prior distribution $\Pi^0$ on the vector of toxicity probabilities, $\bm p = (p_1,\dots,p_K) \in [0,1]^K$, a posterior distribution $\Pi^t$ can be computed by taking into account the first $t$ observations. A possible implementation of Thompson Sampling consists of drawing a sample $\bm \theta(t) = (\theta_1(t),\dots,\theta_K(t))$ from the posterior distribution $\Pi^t$ and selecting at round $t+1$ the dose that is the MTD in the sampled model: $D_{t+1} = \text{argmin}_{k} \ |\theta_k(t) - \theta|$. There are several possible choices for the recommendation rule $\hat k_t$, which are discussed in the upcoming sections.


\subsection{Why Thompson Sampling?} 

Thompson Sampling is by far not the only existing bandit algorithm, yet other algorithms may not be as easily adaptable to the MTD identification problem, which justifies our focus on this algorithm.

Indeed, Thompson Sampling only requires defining some notion of \emph{optimal arm} (or arm to discover), which is naturally defined as the arm with mean closest to the threshold $\theta$ in the MTD identification problem. Many other popular bandit algorithms instead require a \emph{value} to be assigned to each sampled arm, and require the optimal arm to be the arm with largest expected value. This is the case for the frequentist \emph{optimistic} (UCB) algorithms (see, e.g., \cite{Aueral02,KLUCBJournal}), which construct confidence intervals on the expected value of each arm and select the arm which has the largest statistically plausible expected value (i.e. the largest Upper Confidence Bound). Adapting this optimism in face of uncertainty principle for MTD identification is not straightforward: one can certainly build confidence intervals on the toxicity probability of each dose (several of them may contain the MTD), but there is no natural way to define a ``best plausible value'' for each dose in that case. 

In the literature on Bayesian ranking and selection, value-based approaches have also been proposed.  Some algorithms are indeed based on defining some Expected Value of Information \citep{Chick06}. Among those, knowledge gradient methods \citep{Powell12Book} are particularly interesting since they permit handling correlations between arms. For example \cite{XieFrazier16} consider a prior distribution over the arms' means which is a multivariate Gaussian, and \cite{Wang16KGBinary} consider a Bayesian logistic model (where a Laplace approximation is used for Bayesian inference). However, the proposed algorithms are both tailored to finding an arm $a$ maximizing $\bE[V(a,D)]$ for some function $V$ that depends on a random variable $D$ under which the expectation is taken (like other algorithms from the Bayesian Optimization (BO) literature \citep{Brochu10Tuto}). The MTD identification problem cannot naturally be cast in this framework, and adapting, e.g., knowledge gradient methods would require defining an appropriate notion of value of information in this setting. This is why we focused on a Bayesian approach which is easier to adapt to MTD identification, Thompson Sampling.  

\section{Independent Thompson Sampling: an Asymptotically Optimal Algorithm} \label{sec:Analysis}

Inspired by the bandit literature, we introduce the simplest version of Thompson Sampling, that assumes independent uniform prior distributions  on the probability of toxicity of each dose. We refer to this algorithm as {Independent Thompson Sampling} and propose some theoretical guarantees for it.

\subsection{Algorithm Description}

The prior distribution on $\bm{p} = (p_1,\dots,p_K)$ is $\Pi^0 = \bigotimes_{i=1}^{K} \pi_k^0$, where $\pi_k^0 = \cU([0,1])$ is a uniform distribution. Letting $\pi_k^t$ be the posterior distribution of $p_k$ given the observations from the first $t$ patients, the posterior distribution also has a product form, $\Pi^t =\bigotimes_{i=1}^{K} \pi_k^t$. Moreover, each $\pi_k^t$ can be made explicit: $\pi_k^t$ is a $\text{Beta}(S_k(t) + 1, N_k(t) - S_k(t)+1)$ distribution where $S_k(t) = \sum_{s=1}^t X_s \ind_{(D_s = k)}$ is the sum of rewards obtained from arm $k$ and $D_s$ is the dose allocated at time $s$. 

The selection rule of Independent Thompson Sampling is simple: a sample from the posterior distribution on the toxicity probability of each dose is generated, and the dose for which the sample is closest to the threshold is selected: 
\[\left\{\begin{array}{cl}
& \forall k \in \{1,K\}, \  \theta_k(t) \sim \pi_k^t \\
& D_{t+1} = \text{argmin}_{k} \ |\theta_k(t) - \theta|.
\end{array}\right.\]
Several recommendation rules may be used for Independent Thompson Sampling. As the randomization induces some exploration, recommending $\hat{k}_t = D_{t+1}$ is not a good idea. Inspired by what is proposed by 
\cite{Bubeckal11} for assigning a recommendation rule to rewards maximizing algorithms, a first idea is to recommend $\hat{k}_t= \text{argmin}_{k} \ |\hat{\mu}_k(t) - \theta|$, where $\hat{\mu}_k(t)$ is the empirical mean of dose $k$ after the $t$-th patient of the study. Leveraging the fact that TS is supposed to allocate the MTD most of the time, we could also select $\hat{k}_t= \text{argmax}_{k} \ N_k(t)$ or pick $\hat{k}_t$ uniformly at random among the allocated doses.

\subsection{Upper Bound on the Number of Sub-Optimal Selections}

For the classical rewards maximization problem, the first finite-time analysis of Thompson Sampling for Bernoulli bandits dates back to \cite{AGCOLT12} and was further improved by \cite{ALT12,AGAISTAT13}. 
In Appendix~\ref{proof:TS}, building on the analysis of \cite{AGAISTAT13}, we prove the following for Thompson Sampling applied to MTD identification. 

\begin{theorem}\label{thm:TS}  Introducing for every $k\neq k^*$ the quantity  
\[d_k^* := \argmin{d \in \{p_{k^*},2\theta - p_{k^*}\}} \ |p_k - d |,\]
Independent Thompson Sampling satisfies the following. For all $\varepsilon >0$, there exists a constant $C_{\varepsilon,\theta,\bm p}$ (depending on $\varepsilon$, the threshold $\theta$ and the toxicity probabilities) such that for all $k : |p_k - \theta| \neq |\theta - p_{k^*}|$, 
\[
\bE[N_{k}(n)]  \leq \frac{1 + \varepsilon}{\mathrm{kl}(p_k,d_k^*)} \log(n) + C_{\varepsilon,\theta,\bm p},
\]
where $\mathrm{kl}(x,y) = x\log(x/y)+(1-x)\log((1-x)/(1-y))$ is the binary Kullback-Leibler divergence.
\end{theorem}

Theorem~\ref{thm:TS} shows that the total number of allocations to a sub-optimal dose in a trial involving $n$ patients is logarithmic in $n$, which justifies that the MTD is given most of the time, at least in a regime of large values of $n$ (as the second order term can be large). Also, this bounds tells us that in this regime each sub-optimal dose is allocated in inverse proportion of $\kl(p_k,d_k^*)$, which can be seen as a distance between dose $k$ and an optimal dose with toxicity probability $d_k^*$ which is illustrated in Figure~\ref{fig:doses}.  

\begin{figure}\centering
 \includegraphics[height=6cm,angle=-90]{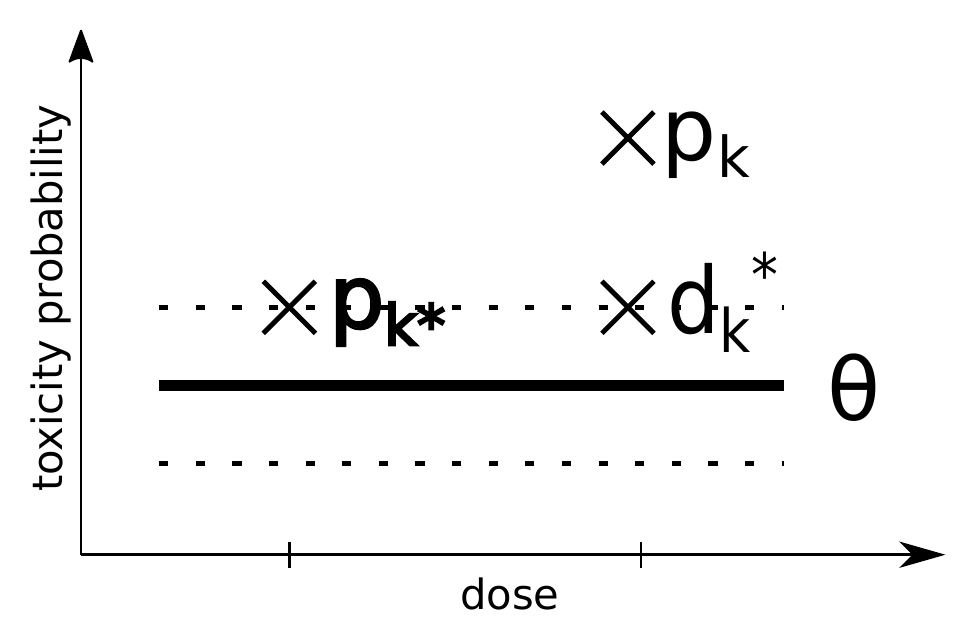}\hspace{1.2cm}
  \includegraphics[height=6cm,angle=-90]{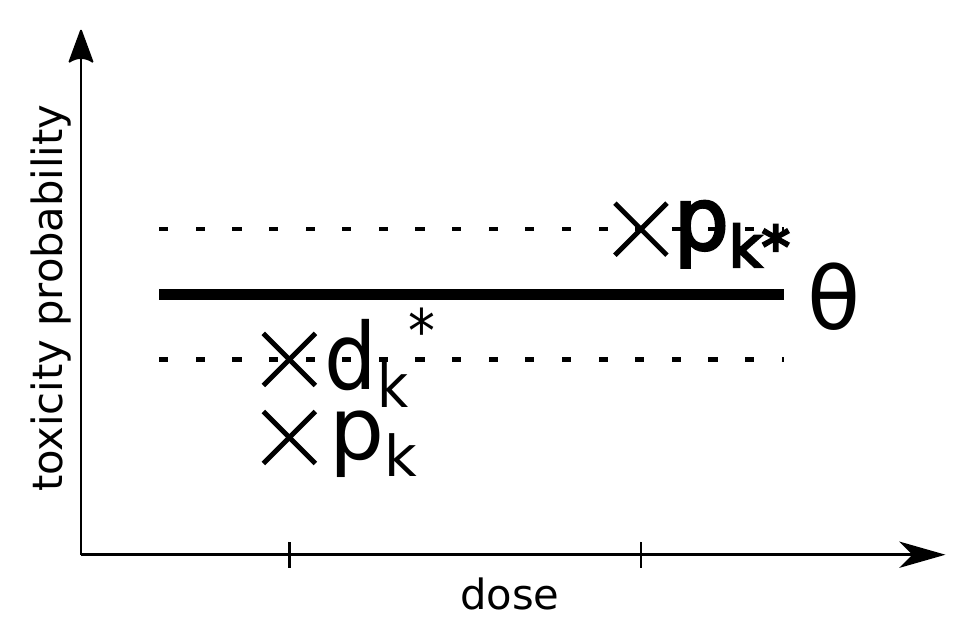}
  \caption{\label{fig:doses}
  Optimal dose $d_k^*$ associated with dose $k$.
  In some cases $d_k^*=p_{k^*}$ (left),
  in others $d_k^* = 2\theta - p_{k^*}$ (right),
  which is symmetric to the MTD with respect to threshold $\theta$.
  }
\end{figure}

The lower bound given in Theorem~\ref{thm:LB} below furthermore shows that Independent Thompson Sampling actually achieves the \emph{minimal number of sub-optimal allocations} when $n$ grows large.

\begin{theorem}\label{thm:LB} We define a uniformly efficient design as a design satisfying for all possible toxicity probabilities $\bm p$, for all $\alpha \in ]0,1[$, for all $k : |\theta - p_k| \neq |\theta - p_{k^*}|$, $\bE[N_k(n)] = o(n^\alpha)$ when $n$ goes to infinity. If $p_{k^*} \neq \theta$, any uniformly efficient design satisfies, for all $k$: $|\theta - p_k| \neq |\theta - p_{k^*}|$, 
\begin{align*}
\liminf_{n \rightarrow \infty} &\frac{\bE[N_k(n)]}{\log(n)}
	\geq \frac{1}{\mathrm{kl}(p_k,d_k^*)}.
\end{align*}
\end{theorem}

Theorem~\ref{thm:LB} can be viewed as a counterpart of the Lai and Robbins lower bound for classical bandits \citep{LaiRobbins85bandits} 
and can be easily derived using recent change-of-measure tools (see \cite{GMS18}). Its proof is given in Appendix~\ref{proof:LB} for the sake of completeness.  

\subsection{Upper Bound on the Error Probability}

If the recommendation rule $\hat{k}_n$ consists of selecting uniformly at random a dose among the doses that were allocated during the trial,  $\{D_1,\dots,D_n\}$, it follows from Theorem~\ref{thm:TS} that 
\begin{equation}\bP\left(\hat k_n \neq k^*\right) = \sum_{k \neq k^*} \frac{\bE[N_k(n)]}{n} \leq \frac{D\ln(n)}{n},\label{UpperBoundError}\end{equation}
where $D$ is a (possibly large) problem-dependent constant. Hence finite-time upper bounds on the number of sub-optimal selection lead to \emph{non-asymptotic upper bound on the error probability} of the design. Note that for the state-of-the-art dose-finding designs it is not known whether such results can be obtained; the only results available provide conditions for \emph{consistency}. For example \cite{ShenOQuigley96,CheungChappell02} exhibit some conditions on the toxicity probabilities under which a classical design called the CRM is such that $\hat{k}_n$ converges almost surely to $k^*$. 

This being said, the upper bound \eqref{UpperBoundError} is not very informative, as a very large number of patients is needed for the upper bound to be at least smaller than 1, and one could expect to have an upper bound that is exponentially decreasing with $n$. As we shall see in Section~\ref{sec:Discussion}, an adaptation of a best arm identification algorithm \citep{Karnin13} leads to such an upper bound, but may be less desirable for clinical trials from an ethical point of view. This is why we rather chose to investigate in what follows several variants of Thompson Sampling coupled with an appropriate recommendation rule.

By using uniform and independent priors on each toxicity probability, Independent Thompson Sampling is the simplest possible implementation of Thompson Sampling. We now explain that using a more sophisticated prior distribution allows the algorithm to leverage some particular constraints of the dose-finding problem, like increasing toxicities or a plateau of efficacy.

\section{Exploiting Monotonicity Constraints with Thompson Sampling}\label{sec:TS}

Independent Thompson Sampling is an adaptation of a state-of-the-art bandit algorithm for identifying the MTD that does not leverage any prior knowledge on (e.g.) the ordering of the arms' means. While it can be argued that when testing drug combinations no natural ordering between the doses exists (see, e.g., \cite{Mozgunov17CT}), in most cases monotonicity assumptions can speed up learning. 

\smallskip

A typical assumption in phase I studies is that both efficacy and toxicity are increasing with the dose. We show in Section~\ref{sec:TSIncreasing} that Thompson Sampling using an appropriate prior is competitive with state-of-the-art phase I approaches leveraging the monotonicity. In Section \ref{sec:TSEff}, we further show that Thompson Sampling is a flexible method that can be useful in phase I/II trials, under more complex monotonicity assumptions on both toxicity an efficacy. More specifically, we show it can handle an efficacy ``plateau,'' where efficacy may be non-increasing after a certain dose level.   

\subsection{Thompson Sampling for Increasing Toxicities: A Phase I Design}\label{sec:TSIncreasing}

In a phase I study in which both toxicity and efficacy are increasing with the dose, the MTD is the most relevant dose to allocate in further stages. We now focus on algorithms leveraging the extra information that $p_1 \leq \dots \leq p_k$. 
To exploit this structure, \emph{escalation procedures} have been developed in the literature, the most famous being the ``3+3'' design \citep{storer89}. In this design, adjusted for $\theta = 0.33$, the lowest dose is first given to 3 patients. If no patient experiences toxic effects, one escalates to the next dose and repeats the process. If one patient experiences toxicity, the dose is given to 3 more patients, and if less than two patients among the 6 experience toxicity, one escalates to the next dose. Otherwise the trial is stopped, which is also the case if from the beginning 2 out of the 3 patients experience a toxic effect. Upon stopping, the previous dose is recommended as the MTD, or all doses are deemed too toxic if one stops at the first dose level. Although it is clear that the guarantees in terms of error probability (or sub-optimal selections) are very weak, ``3+3'' is still often used in practice.

Alternative to this first design are variants of the Continuous Reassessment Method (CRM), proposed by \cite{OQuigley90CRM}. The CRM uses a Bayesian model that combines a parametric dose/toxicity relationship with a prior on the model parameters. Under this model, CRM appears as a greedy strategy that selects in each round the dose whose expected toxicity under the posterior distribution is closest to the threshold. We propose in this section several variants of Thompson Sampling based on the same Bayesian model, but that favor (slightly) more exploration.

\paragraph{A Bayesian model for increasing toxicities} In the CRM literature, several parametric models that yield an increasing toxicity have been considered. In this paper, we choose a two-parameter logistic model that is among the most popular. Under this model, each dose $k$ is assigned an \emph{effective dose} $u_k$ (that is usually not related to a true dose expressed in a mass or volume unit) and the toxicity probability of dose $k$ is given by 
\begin{align*}
p_k(\beta_0,\beta_1) &= \psi(k,\beta_0,\beta_1), \ \ \text{ where } \ \
\psi(k,\beta_0,\beta_1) = \frac{1}{1 + e^{-\beta_0-\beta_1u_k}}.
\end{align*}
A typical choice of prior is 
\[
\beta_0 \sim \mathcal{N}(0,100) \ \ \text{and} \ \ \beta_1 \sim \mathrm{Exp}(1). \]
It is worth noting that this model also heavily relies on the distinct effective dose levels $u_1,\dots,u_K$ that are usually chosen depending on some \emph{prior toxicities} set by physicians, $p^0_1 \leq p^0_2 \leq \dots \leq p^0_K$. Letting $\overline{\beta_0}$, $\overline{\beta_1}$ be the prior mean of each parameter, the effective doses are calibrated such that for all $k$, $\psi(k,\overline{\beta_0},\overline{\beta_1}) = p_k^0$. If there is no medical prior knowledge about the toxicity probabilities, some heuristics for choosing them in a robust way have been developed (see Chapter 9 of \cite{CRMBook}).
  
Under this model, given some observations from the different doses one can compute the posterior distribution over the parameters $\beta_0$ and $\beta_1$; that is, the conditional distribution of these parameters given the observations. Although there is no closed form for these posterior distributions, they can be easily sampled from using Hamiltonian Monte-Carlo Markov Chain algorithms (HMC) as the log-likelihood under these models is differentiable. In practice, we use the Stan implementation of these Monte-Carlo sampler \citep{StanManual}, and use (many) samples to approximate integrals under the posterior when needed.    
 
\subsubsection{Thompson Sampling}\label{sec:CRMvsTS}

Thompson Sampling selects a dose at random according to its posterior probability of being the MTD. Under the two-parameter Bayesian logistic model presented above, letting $\pi_t$ denote the posterior distribution on $(\beta_0,\beta_1)$ after the first $t$ observations, the posterior probability that dose $k$ is the MTD is
\begin{align*}
{q}_k(t) &: = \bP\left(\left. k = \argmin{\ell} |\theta - p_\ell(\beta_0,\beta_1)| \right| \cF_t\right)
\\
&= \int_{\R} \ind{\left(k = \argmin{\ell} |\theta - p_\ell(\beta_0,\beta_1)|\right)} d\pi_t(\beta_0,\beta_1).
\end{align*}

A first possible implementation of Thompson Sampling that we use in our experiments consists of computing approximations $\hat{q}_k(t)$ of the probabilities ${q}_k(t)$ (using posterior samples) and selecting at round $t+1$ a dose $D_{t+1} \sim \hat{\bm q}(t)$, i.e. such that $\bP\left(D_{t+1} = k | \cF_t\right) = \hat{q}_k(t)$.
A second implementation of Thompson Sampling (that may be computationally easier) consists of drawing one sample from the posterior distribution of $(\beta_0,\beta_1)$, and selecting the MTD in the sampled model:  
\begin{align}
         \left(\tilde \beta_0(t), \tilde \beta_1(t)\right) & \sim \pi_t, \nonumber\\
          D_{t+1}^{\text{TS}} & \in \argmin{k \in \{1,\dots,K\}} \ \left|\theta - p_k\left(\tilde \beta_0(t),\tilde \beta_1(t)\right)\right|.\label{eq:SampleTS}
\end{align}
It is easy to see that this algorithm coincides with Thompson Sampling in that $\bP\left(D_{t+1}^{\text{TS}} = k | \cF_t\right) = {q}_k(t)$. We will present below a variant of Thompson Sampling based on the first implementation (${\mathrm{TS}\_\mathrm{A}}$) and a variant based on the second implementation (${\mathrm{TS}(\varepsilon)}$).

\paragraph{Recommendation rule} Due to the randomization, Thompson Sampling performs more exploration than the ``greedy'' CRM \citep{OQuigley90CRM} method, which selects at time $t$ the MTD under the model parameterized by $(\hat\beta_0,\hat\beta_1)$, the posterior means of the two parameters, given by
\begin{equation}\label{eq:CRMPostMean}
\hat{\beta}_0(t) = \int_\R \beta_0 d\pi_{t}(\beta_0,\beta_1) \ \ \ \ \ \text{and} \ \ \ \ \ \hat{\beta}_1(t) = \int_\R \beta_1 d\pi_{t}(\beta_0,\beta_1).
\end{equation}
More precisely, the sampling rule of the CRM is
\begin{align*}
D_{t+1}^{\text{CRM}} \in \argmin{k \in \{1,\dots,K\}} \left|\theta - p_k(\hat\beta_0(t),\hat\beta_1(t))\right|.
\end{align*}
The recommendation rule for CRM after $t$ patients is identical to the next dose that would be sampled under this design, that is $\hat{k}_t^{\text{CRM}} = D_{t+1}^{\text{CRM}}$. For Thompson Sampling, due to the more exploratory nature of the algorithm, we do not want to recommend $\hat{k}_t^{\text{TS}} = D_{t+1}^{\text{TS}}$. Instead, we propose the use of recommendation rule $\hat{k}_t^{\text{TS}} = \underset{k \in \{1,\dots,K\}}{\text{argmin}} \ |\theta - p_k(\hat\beta_0(t),\hat\beta_1(t))|$, which coincides with that of the CRM.

\subsubsection{Two variants of Thompson Sampling}\label{subsec:ParameterTuning}

The randomized aspect of Thompson Sampling makes it likely to sample from large or small doses, without respecting some ethical constraints of phase I clinical trials. Indeed, patients should not be exposed to too-high dose levels; overdosing should be controlled. Hence, we also propose two ``regularized'' versions of TS. The first depends on a parameter $\varepsilon>0$ set by the user that ensures that the expected toxicity of the recommended dose remains within $\varepsilon$ of the toxicity of the empirical MTD. The second restricts the doses to be tested to a set of \emph{admissible doses}. These algorithms are formally defined below, and their performance is evaluated in Section~\ref{sec:Experiments}.

\paragraph{$\bm{\mathrm{TS}(\varepsilon)}$}  We first compute the posterior means $\hat\beta_0(t),\hat{\beta}_1(t)$ from \eqref{eq:CRMPostMean} and the toxicity of the dose closest to $\theta$ under the model parameterized by $(\hat\beta_0(t),\hat{\beta}_1(t))$ (i.e., the toxicity of the dose selected by the CRM): 
\[
\hat{p}(t) = p_{\hat k_{t}}(\hat\beta_0(t),\hat\beta_1(t)), \ \ \
\text{with} \ \ \ \hat k_{t} =\argmin{k \in \{1,\dots,K\}}\left|\theta - p_k(\hat\beta_0(t),\hat\beta_1(t))\right| 
\]
Next we sample $\tilde \beta_0(t), \tilde{\beta}_1(t)$ from the posterior distribution $\pi_t$
and select a candidate dose level $D_{t+1}$ using \eqref{eq:SampleTS}.
If the predicted toxicity level $p_{D_{t+1}}(\hat\beta_0(t),\hat{\beta}_1(t))$ is not in the
interval $(\hat{p}(t)-\varepsilon, \hat{p}(t)+\varepsilon)$, then we reject
our values of $\tilde\beta_0(t),\tilde{\beta}_1(t)$, draw a new sample from $\pi_t$ and repeat the process. 
In order to guarantee that the algorithm terminates, we only reject up to
50 samples, after which we use the sample that gives the dose with minimum toxicity among all $50$ samples. We choose 50 to limit the computational complexity of the algorithm, but it can also be replaced by a larger value if more computational power is available.

$\bm{\mathrm{TS}(\varepsilon)}$ can be seen as a smooth interpolation between the CRM (which correspond to $\varepsilon = 0$) and vanilla Thompson Sampling (which corresponds to $\varepsilon = 1$). Regarding the tuning of the parameter $\varepsilon$, large values do not reduce much the amount of exploration while too small values lead to a behavior which is indistinguishable from that of the CRM. We did (large scale) experiments with $\varepsilon \in \{0.02, 0.05, 0.1\}$ and we found that the three values lead to comparable performance across the different scenarios we tried. To ease the presentation, we report results for TS($0.05$) only in Section~\ref{sec:Experiments}.

\paragraph{$\bm{\mathrm{TS}\_\mathrm{A}}$} The $\mathrm{TS}\_\mathrm{A}$ algorithm  limits exploration by enforcing the selected dose to be in some admissible set $\cA_t$, by sampling from the modified distribution 
\[\bP\left(D_{t+1} = k | \cF_{t}\right) = \frac{\hat{q}_k(t) \ind_{\left(k \in \cA_t\right)}}{\sum_{\ell \in \cA_t} \hat{q}_{\ell}(t)},\] 
instead of sampling directly from $\bm{\hat{q}}(t)$ as vanilla Thompson Sampling does. 
The admissible set  $\cA_t$ is defined as set of  doses that meet the following two criteria:
\begin{enumerate}
    \item dose $k$ has either already been tested, or is the next-smallest dose which has not yet been tested    
    \item the posterior probability that the toxicity of dose $k$ exceeds the toxicity of the dose closest to $\theta$ is smaller than some threshold:
\[
   \bP\Bigg(
    \psi(k,\beta_0,\beta_1) > \psi(k',\beta_0,\beta_1), 
    \text{ where } k' = \argmin{k' \in \{1,\dots,K\}}\left|\theta - \psi(k', \beta_0,\beta_1)\right|
    \Bigg| \cF_{t} \Bigg) \leq c_1.
\]
%
\end{enumerate}
$\cA_t$ is inspired by the admissible set of \cite{MKR17} described in detail in the next section.

In our experiments, we tried different values of the parameter $c_1$ and we found that the performance of TS$\_$A is better with values of $c_1$ that are not too small. In Section~\ref{sec:Experiments}, we report experiments with $c_1 = 0.8$, but the performance was comparable for the choices $c_1 = 0.6$ or $0.9$.

\subsection{Thompson Sampling for Efficacy Plateau Models: A Phase I/II Design}\label{sec:TSEff}


In some particular trials, it has been established that efficacy is not always increasing with the dose. Motivated by some concrete examples discussed in their paper, \cite{MKR17} consider a model in which the dose effectiveness can plateau after some unknown level, while toxicity still increases with dose level. In these models, MTD identification is no longer relevant and the objective is rather to identify the smallest dose with maximal efficacy and with toxicity no more than $\theta$. More formally, introducing $\text{eff}_k$ the efficacy probability of dose $k$, the Minimal Effective Dose (MED) is 
\[k^* = \min\left\{ k : \eff_k = \max_{\ell : p_\ell \leq \theta} \ \text{eff}_\ell\right\}  \]

In a dose-finding study involving efficacy, at each time step $t$ a dose $D_t$ is allocated to the $t$-th patient, and the toxicity $X_t$ is observed, as well as the efficacy $Y_t$. With this two-dimensional observation, assigning a value (or reward) to each sampled arm is even less natural than before. However as one can still define a notion of optimal dose (the MED instead of the MTD), Thompson Sampling  can still be applied in this setting. As we shall see, it bears some similarities to the state-of-the-art method developed by \cite{MKR17}.

\paragraph{A Bayesian model for toxicity and efficacy} Thompson Sampling requires a Bayesian model for both the dose/toxicity and the dose/efficacy relationship that enforces an increasing toxicity and a increasing then plateau efficacy. We use the model proposed by \cite{MKR17}, that we now describe.

Under this model, toxicity and efficacy are assumed to be independent. The (increasing) toxicity follows the two-dimensional Bayesian logistic model with effective doses $u_k$:  
\begin{align*}
p_k &= p_k(\beta_0,\beta_1) = \psi(k,\beta_0,\beta_1) 
\\
\text{and} \ \ \ \beta_0 &\sim \mathcal{N}(0,100), \ \ \ \beta_1 \sim \text{Exp}(1).
\end{align*}
Efficacy also follows a logistic model, with an additional parameter $\tau$ that indicates the beginning of the plateau of efficacy. The efficacy probability of dose level $k$ is
\[
\eff_k = \eff_k(\gamma_0,\gamma_1,\tau) = \phi(k,\gamma_0,\gamma_1,\tau),  \
\text{ where }
    \ \phi(k,\gamma_0,\gamma_1,\tau) := \frac{1}{1 + e^{-\left[\gamma_0 + \gamma_1(
    v_k \mathds{1}(k<\tau) + v_{\tau} \mathds{1}(k\ge\tau) )\right]}},
\]
with $v_k$ the \emph{effective efficacy} of dose $k$. Given $(t_1,\dots,t_K)$ such that $\sum_{i=1}^K t_i = 1$, a probability distribution on $\{1,\dots,K\}$, the three parameters $(\gamma_0,\gamma_1,\tau)$ are independent and drawn from the following prior distribution:     
\[\gamma_0 \sim \mathcal{N}(0,100), \ \ \ \gamma_1 \sim \text{Exp}(1), \ \ \  \tau  \sim (t_1,\dots,t_K).\]
The prior on $\tau$ may be provided by a physician or set to $(1/K,\dots,1/K)$ in case one has no prior information. Just like the effective doses $u_k$ (that we may now call effective toxicities), the effective efficacies $v_k$ are calculated using prior efficacies $\eff^0_1 \leq \dots \leq \eff^0_K$:
\begin{align*}
v_k &= \left( \log\left( \frac{ \eff^0_k }{1 - \eff^0_k} \right)
    - \overline{\gamma}_0 \right) \bigg/ \overline{\gamma}_1,
\end{align*}
where $\overline{\gamma}_0 = 0$ and $\overline{\gamma}_1 = 1$ are the prior means of the parameters $\gamma_0$ and $\gamma_1$. 


\paragraph{Posterior sampling} Let
$\mathcal{D}^{\eff}_t = \{(D_1,Y_1),\dots,(D_t,Y_t)\}$
be the efficacy data gathered in the first $t$ rounds.
Generating samples from the posterior distribution of $(\gamma_0,\gamma_1,\tau)$ given $\mathcal{D}^{\eff}_t$ is a bit more involved than generating posterior samples from $(\beta_0,\beta_1)$. Indeed, it cannot be handled directly with HMC given that $(\gamma_0,\gamma_1)$ are continuous and $\tau$ is discrete. Thus, we proceed in the following way: we first draw samples from $p(\gamma_0,\gamma_1 |\mathcal{D}^{\eff}_t)$, which can be performed with HMC (and requires marginalizing out the discrete parameter $\tau$, following the example of change point models given in the Stan manual \citep{StanManual}). Then we sample $\tau$ conditionally to $\gamma_0,\gamma_1,\mathcal{D}^{\eff}_t$.

\subsubsection{Thompson Sampling}

Recall that the principle of Thompson Sampling is to randomly select doses according to their posterior probability of being optimal. This idea can also be applied in this more complex model, using the corresponding definition of optimality. Given a vector $\bm\psi = (\psi_1,\dots,\psi_K)$ of increasing toxicity probabilities and a vector $\bm\phi = (\phi_1,\dots,\phi_K)$ of increasing then plateau efficacy probabilities, the optimal dose is 
\[
{\Opt}(\bm{\psi},\bm{\phi}) : = \min\left\{
    k : \phi_k = \max_{\ell : \psi_\ell \leq \theta} \phi_\ell \right\}.
\]

The posterior probability of dose $k$ to be optimal in that case is 
\[{q}_k(t) := \bP\left( \left.k = \Opt\left( \psi(\bm\cdot, \beta_0,\beta_1),\phi(\bm\cdot, \gamma_0,\gamma_1,\tau)\right) \right|\cF_{t}\right)\]
and in our experiments, we implement Thompson Sampling by computing approximations $\hat{q}_k(t)$ from the quantities $q_k(t)$ (based on posterior samples) and then selecting a dose $D_{t+1}\sim \bm{\hat{q}}(t)$ where  $\bm{\hat{q}}(t) = (\hat{q}_1(t),\dots,\hat{q}_K(t))$. Just like in the previous model, an alternative implementation of Thompson Sampling would sample parameters from their posterior distributions and select the optimal dose in this sampled model. Letting 
\[\tilde{\beta}_0(t), \tilde{\beta}_1(t) \ \ \ \text{and} \ \ \ \tilde{\gamma}_0(t), \tilde{\gamma}_1(t), \tilde{\tau}(t),\]
be samples from the posterior distributions after $t$ observations of the toxicity and efficacy parameters respectively, one can compute $\tilde{\psi}_k(t) = \psi(k,\tilde{\beta}_0(t),\tilde{\beta}_1(t))$ and $\tilde{\phi}_k(t) = \phi(k,\tilde{\gamma}_0(t),\tilde{\gamma}_1(t),\tilde{\tau}(t))$ for every dose $k$. Given the toxicity and efficacy vectors  
\begin{align*}
\bm{\tilde{\psi}}(t) &= \left(\tilde{\psi}_1(t),\dots,\tilde{\psi}_K(t)\right) \\ \text{ and} \ \ \bm {\tilde{\phi}}(t) &= \left(\tilde{\phi}_1(t),\dots,\tilde{\phi}_K(t)\right),
\end{align*}
this implementation of Thompson Sampling selects at round $t+1$ $D_{t+1}^{\text{TS}} = \Opt\left(\bm{\tilde{\psi}}(t), \bm{\tilde{\phi}}(t)\right)$.

\paragraph{Recommendation rule} Here also we expect Thompson Sampling to be too exploratory for dose recommendation. Hence, we base our recommendation on estimated values. Given the posterior means 
$\hat{\beta}_0(t), \hat{\beta}_1(t),\hat{\gamma}_0(t),\hat{\gamma}_1(t)$ (estimated from posterior samples) and $\hat{\tau}(t)$ the mode of the posterior distribution of the breakpoint (see the next section for its computation), we compute  
$\hat{\psi}_k(t) = \psi(k,\hat{\beta}_0(t),\hat{\beta}_1(t))$ and $\hat{\phi}_k(t) = \phi(k,\hat{\gamma}_0(t),\hat{\gamma}_1(t),\hat{\tau}(t))$ and recommend $\hat{k}_t = \Opt\left(\bm{\hat{\psi}}(t), \bm{\hat{\phi}}(t)\right)$.

\subsubsection{A Variant of Thompson Sampling using Adaptive Randomization}

Interestingly, the need for randomization in the context of plateau efficacy has already been observed by \cite{MKR17}. More precisely, as we explain below, the algorithm $\mathrm{MTA}$-$\mathrm{RA}$ described in that work can be viewed as an hybrid approach between Thompson Sampling and a CRM approach. 

Additionally to the use of \emph{adaptive randomization}, the $\mathrm{MTA}$-$\mathrm{RA}$ algorithm also introduces a notion of  \emph{admissible set}. The set of admissible doses after $t$ patients, denoted by $\cA_t$, is the set of dose levels $k$ meeting all of the following criteria:
\begin{enumerate}
    \item dose $k$ has either already been tested, or is the next-smallest dose which has not yet been tested
    \item the posterior probability that the toxicity of dose $k$ exceeds $\theta$ is smaller than some threshold: 
    \begin{equation}\bP\left(\psi(k,\beta_0,\beta_1) > \theta | \cF_{t}\right) \leq c_1\label{eq:CritTox}\end{equation}
    \item if the dose has been tested more than $3$ times, the posterior probability that the efficacy is larger than $\xi$ is larger than some threshold: 
    \begin{equation}\bP\left(\phi(k,\gamma_0,\gamma_1,\tau) > \xi | \cF_{t}\right) \geq c_2\label{eq:CritEff}\end{equation}
\end{enumerate}
Practical computation of the admissible set can be performed using posterior samples from $(\beta_0,\beta_1)$ to check the criterion \eqref{eq:CritTox} and posterior samples from $(\gamma_0,\gamma_1,\tau)$ to check the criterion \eqref{eq:CritEff}. 

The $\mathrm{MTA}$-$\mathrm{RA}$ algorithm works in two steps. The first step exploits the \emph{posterior distribution of the breakpoint}, $t_k(t) :=\bP\left(\tau=k | \cD_{t}^{\eff}\right)$, and uses randomization to pick a value $\hat{\tau}(t)$ close to the mode of this distribution. More precisely, given $(\hat{t}_k(t))_{k=1,\dots,K}$ an estimate of the posterior distribution of $\tau$, let
\[
\mathcal{R}_t := \left\{
    k : \left| \max_{1 \le \ell \le K}(\hat{t}_\ell(t)) - \hat{t}_k(t) \right| \le s_1,
    1 \le k \le K
\right\}
\]
be a set of candidate values for the position of the breakpoint. Then under $\mathrm{MTA}$-$\mathrm{RA}$, 
\[\bP\left(\hat{\tau}(t) = k |\cF_t\right) = \frac{\hat{t}_k(t)\ind_{\left(k \in \cR_t\right)}}{\sum_{\ell \in \cR_t} \hat{t}_\ell(t)}. \]
The threshold $s_1$ is often adapted such that it is larger in the beginning of the trial when we have high uncertainty about the estimates, but it grows smaller as the trial continues. The second step of $\mathrm{MTA}$-$\mathrm{RA}$ doesn't employ randomization. Based on posterior samples from $(\gamma_0,\gamma_1)$ conditionally to $\tau$ being equal to $\hat{\tau}(t)$, efficacy estimates $\hat{\phi}_k$ are produced (taking the mean of the values of $\phi(k,\tilde{\gamma_0},\tilde{\gamma_1},\hat{\tau}(t))$ for many samples $\tilde{\gamma_0},\tilde{\gamma_1}$) and finally the selected dose is 
\[ D_{t+1}^{\text{MTA-RA}} = \inf \left\{ k \in \mathcal{A}_t : \hat{\phi}_k = \max_{j \in \mathcal{A}_t} \hat{\phi}_j\right\}.\]

If $\hat{\tau}(t)$ were replaced by a point estimate (e.g. the mode of the breakpoint posterior distribution $\bm{\hat t}(t)$), MTA-RA would be close to a CRM approach that computes estimates of all the parameters and acts greedily with respect to those estimated parameters (with the additional constraint that the chosen dose has to remain in the admissible set). However, the first step of MTA-RA bears similarities with the first step of a Thompson Sampling implementation that would sample a parameter $\tau$ from the $\bm{\hat t}(t)$ (and later sample the other parameters conditionally to that value and act greedily in the sampled model). The difference is the use of \emph{adaptive} randomization, in which the sample is not exactly drawn from $\bm{\hat t}(t)$, but is constrained to fall in some set (here $\cR_t$) that depends on previous observations. 

\paragraph{The $\bm{\mathrm{TS}\_\mathrm{A}}$ algorithm} We believe that using adaptive randomization is a good idea to control the amount of exploration performed by Thompson Sampling, which leads us to propose the $\mathrm{TS}\_\mathrm{A}$ algorithm, that incorporates the constraint to select a dose that belongs to the admissible set $\cA_t$. More formally, $\mathrm{TS}\_\mathrm{A}$ selects a dose at random according to   
\[\bP\left(D_{t+1} = k | \cF_{t}\right) = \frac{\hat{q}_k(t) \ind_{\left(k \in \cA_t\right)}}{\sum_{\ell \in \cA_t} \hat{q}_{\ell}(t)},\]
where we recall that $\hat{q}_k(t)$ is an estimate of the posterior probability that dose $k$ is optimal. Compared to the variant of $\mathrm{TS}\_\mathrm{A}$  for increasing toxicities that is proposed in Section~\ref{sec:TSIncreasing}, the difference here is the appropriate definition of the admissible set, that involves both toxicity and efficacy probabilities.

\paragraph{Practical remark} Approximations $\hat{t}_k(t)$ of the breakpoint distribution can be computed using that
\[
t_k(t) = {t_k \int \frac{L(\mathcal{D}^{eff}_t | \gamma_0,\gamma_1,k)}{\sum_{s=1}^K t_s L(\mathcal{D}^{eff}_t | \gamma_0,\gamma_1,s)} p(\gamma_0,\gamma_1 | \mathcal{D}^{eff}_t) d\gamma_0 d\gamma_1},
\]
where $L(\mathcal{D}^{\text{eff}}_t | \gamma_0,\gamma_1,s)$ is the likelihood of the efficacy observations when the efficacy model parameters are $(\gamma_0,\gamma_1,s)$ and $p(\gamma_0,\gamma_1 | \mathcal{D}^{\text{eff}}_t)$ is the density of the distribution of $(\gamma_0,\gamma_1)$ given the observations. $\hat{t}_k(t)$ can be thus be obtained by Monte-Carlo estimation based on samples from $p(\gamma_0,\gamma_1 | \mathcal{D}^{\text{eff}}_t)$.

\section{Experimental Evaluation}\label{sec:Experiments}

We now present an empirical evaluation of the variants of Thompson Sampling introduced in the paper first in the context of increasing efficacy and then with the presence of a plateau of efficacy. In both groups of experiments, we adjusted our designs to some common practices in dose-finding trials. We used a start-up phase for all designs (starting from the smallest dose and escalating until the first toxicity is observed) and we also used cohorts of patients of size 3. This means that the same dose is allocated to 3 patients at a time and the model is updated after seeing the outcome for these 3 patients. 

\subsection{Phase I: MTD Identification}

In this set of experiments, we evaluate the performance of the three algorithms introduced in Section~\ref{sec:TSIncreasing}, $\mathrm{TS}$, $\mathrm{TS}(\varepsilon)$ and $\mathrm{TS}\_\mathrm{A}$, and compare them to the 3+3 and $\mathrm{CRM}$ baselines. We report experiments with the value $\varepsilon = 0.05$ for $\mathrm{TS}(\varepsilon)$ and $c_1 = 0.8$ for $\mathrm{TS}\_\mathrm{A}$. We refer the reader to Section~\ref{subsec:ParameterTuning} for discussions on the choice of these parameters. We also include $\mathrm{Independent \ TS}$ as proposed in Section~\ref{sec:Bandits}, which is agnostic to the increasing structure. 
 
In Tables~\ref{tbl-tox} to~\ref{tbl-tox-c} we provide results for nine different scenarios in which there are $K=6$ doses with a target toxicity $\theta = 0.30$, budget $n=36$ and prior toxicities
\[\bm{p^{0}} = [0.06 \ \ 0.12 \ \ 0.20 \ \ 0.30 \ \ 0.40 \ \ 0.50].\]
We choose the same prior toxicities for all scenario, that are sometimes close to actual toxicities (e.g. in Scenario 2) and sometimes quite far, in order to showcase the robustness of Bayesian algorithms.    

For each scenario and algorithm, we report in the first column of these tables the percentage of allocation to each dose, that is, an estimate of $\bP(\hat{k}_n = k)$ for each dose $k$, based on $N=2000$ repetitions. In the second column, we report an estimate of the percentage of allocation to each dose during the trial, computed for each dose $k$ as the average value of $100*N_k(n)/n$ over $N=2000$ repetitions. We add in parenthesis the empirical standard deviation of these allocation percentages, as allocations under bandit algorithms are known to have a large variance. For the 3+3 design, only the recommendation percentages are displayed, as the percentage of allocations would be computed based on a number of patients smaller than 36 (as a 3+3 based trial involves some random stopping). This design is also the only one that would stop and recommend none of the doses if they are all judged too toxic: we add this fraction of no recommendation between brackets in the tables. 

\medskip

For each scenario, corresponding to different increasing toxicity probabilities, the MTD is underlined
and we mark in bold the fraction of recommendation or allocation of the MTD that are superior to what is achieved by the CRM. We now comment on the performance of the algorithms on those scenarios.

\paragraph{Dose recommendation}  $\mathrm{TS}$ outperforms CRM 3 out of 9 times, $\mathrm{TS}(\varepsilon)$ does so 5 out of 9 times, and $\mathrm{TS}\_\mathrm{A}$ does so 5 out of 9 times. As expected, $\mathrm{Independent \ TS}$, which does not leverage the increasing structure, does not have a remarkable performance. This algorithm would need a larger budget to have a good empirical performance. With $n=36$ in most cases this strategy is not doing much better than selecting the doses uniformly at random. One can also observe that the 3+3 design (that may however require less than 36 patients in the trial) performs very badly in terms of dose recommendation.

\paragraph{Dose allocation} While $\mathrm{TS}\_\mathrm{A}$ and $\mathrm{TS}(\varepsilon)$ do not always have higher allocation percentage at the optimal (underlined) dose compared to CRM, a scan of the dose allocation results in Tables~\ref{tbl-tox}~to~\ref{tbl-tox-c} shows that the addition of the admissible set $\cA$ and $\varepsilon$ regularity to the Thompson Sampling method consistently reduces the allocation percentage of higher toxicity doses. $\mathrm{TS}\_\mathrm{A}$ performs best in this regard (it is more cautious with allocating higher doses) across all algorithms (e.g. it consistently has superior performance compared to CRM), while $\mathrm{TS}(\varepsilon)$ has performance better than or comparable to CRM. We believe this result is of interest in trials where toxicity is an ethical concern.

\subsection{Phase I/II: MED Identification when Efficacy Plateaus} \label{subsec:toxOnlyExp}

In this set of experiments, we evaluate the performance of the two algorithms introduced in Section~\ref{sec:TSEff}, $\mathrm{TS}$ and $\mathrm{TS}\_\mathrm{A}$, and compare them to the $\mathrm{MTA}$-$\mathrm{RA}$ algorithm. We use the experimental setup of \cite{MKR17}: several scenarios with $K=6$ doses, budget $n=60$, $\theta = 0.35$, toxicity and efficacy priors
\begin{align*}
\bm{p^0} &= [0.02, 0.06, 0.12, 0.20, 0.30, 0.40] \ \ \text{ and } \ \ \bm{\mathrm{eff}^0} = [0.12, 0.20, 0.30, 0.40, 0.50, 0.59].
\end{align*}
Furthermore, we use the same parameters for the admissible set and the implementation of $\mathrm{MTA}$-$\mathrm{RA}$ as those chosen by \cite{MKR17}: $\xi=0.2$, $c_1=0.9$, $c_2=0.4$, and $s_1=.2\left(1-\frac{I}{n}\right)$, where $I$ is the number of samples used so far. These parameters are defined above in the main text.


In Tables~\ref{tbl-eff}~to~\ref{tbl-eff-b} we provide results for several scenarios with increasing toxicities and efficacy, with efficacy which (quasi) plateaus. We report the percentage of allocation to each dose, the percentage of recommendation of each dose when $n=60$, and the percentage of time the trials stopped early (E-Stop), estimated over $N=2000$ repetitions. As before, we also report standard deviations for the percentage of allocations to each dose. 

Optimal doses are underlined by a plain line while a dashed line identifies doses whose toxicity is larger than $\theta$. We mark in bold cases where
our algorithms makes the optimal decision (in terms of the percentage of recommendations) more often than the
$\mathrm{MTA}$-$\mathrm{RA}$ baseline.

\paragraph{Dose recommendation}  
Recall that the modeling assumption here is that efficacy increases monotonically in toxicity up to a point and then it plateaus. We present experimental results on several scenarios, some of which are borrowed from \cite{MKR17}, on which this plateau assumption is not always exactly met. In most of these scenarios, $\mathrm{TS}\_\mathrm{A}$ outperforms the $\mathrm{MTA}$-$\mathrm{RA}$ algorithm.

In scenarios 1 through 4 and in scenarios 12 and 13, there is a plateau of efficacy starting at a reasonable toxicity: in this case the optimal dose corresponds to the plateau breakpoint. Our algorithms make the optimal decision compared to $\mathrm{MTA}$-$\mathrm{RA}$ consistently: $\mathrm{TS}$ 4 out of 6 times and $\mathrm{TS}\_\mathrm{A}$ 5 out of 6 times. In scenarios 5 and 6 the plateau of efficacy starts when the toxicity is already too high, hence the optimal dose is before than the plateau. In scenario 5,  $\mathrm{TS}\_\mathrm{A}$ and $\mathrm{TS}$ both outperform $\mathrm{MTA}$-$\mathrm{RA}$, while on scenario 6 $\mathrm{MTA}$-$\mathrm{RA}$ has a slight advantage over $\mathrm{TS}$. 


In scenario 7 and 8 there is no true plateau of efficacy, however in both cases there exists a ``breakpoint'' (underlined) after which the efficacy is increasing very slowly while the toxicity is increasing significantly. This breakpoint can thus be argued to be a good trade-off between efficacy and toxicity and should be investigated in further phases. In these two scenarios $\mathrm{TS}\_\mathrm{A}$ identifies this pseudo-optimal dose more often than $\mathrm{MTA}$-$\mathrm{RA}$, while  $\mathrm{TS}$ has a slightly worse performance. 

Lastly, we study the case when there is no clear optimal or near-optimal dose, i.e. scenarios 9-11. In scenario 9 wherein most doses, including the entire quasi-plateau, are too toxic, we would like to stop early or at most recommend dose 1 (the only dose meeting the toxicity constraint but whose efficacy is not very high). Under this interpretation, $\mathrm{TS}$ and $\mathrm{TS}\_\mathrm{A}$ outperform $\mathrm{MTA}$-$\mathrm{RA}$. Note that our algorithms most often either stop early or recommend dose 1, while in comparison $\mathrm{MTA}$-$\mathrm{RA}$ recommends the toxic dose 2 a large fraction of the time (33.1 \%). In scenarios 10 and 11 in which all doses are either too toxic or ineffective a good algorithm would stop early with no recommendation. 
$\mathrm{TS}\_\mathrm{A}$ makes this optimal decision more often than $\mathrm{MTA}$-$\mathrm{RA}$ in both scenarios and $\mathrm{TS}$ in one of the two scenarios.

\paragraph{Dose allocation}
While $\mathrm{TS}$ and $\mathrm{TS}\_\mathrm{A}$ have lower allocation percentage at the optimal (underlined) dose compared to $\mathrm{MTA}$-$\mathrm{RA}$, the addition of the admissible set $\cA$ to the Thompson Sampling method consistently reduces the percentage of dose allocation at doses that are too toxic. Furthermore, $\mathrm{TS}\_\mathrm{A}$ is more cautious in allocating higher doses compared to $\mathrm{MTA}$-$\mathrm{RA}$. Our experiments notably reveal that the fraction of allocation to doses whose toxicity is larger than $\theta$ (that are underlined with a dashed line) is always smaller for $\mathrm{TS}\_\mathrm{A}$ than for $\mathrm{MTA}$-$\mathrm{RA}$. Hence, not only is $\mathrm{TS}\_\mathrm{A}$ very good in terms of recommending the right dose, it also manages to avoid too-toxic doses more consistently.

\begin{table*}[p]
\caption{Results for MTD identification}\label{tbl-tox}
\renewcommand{\tabcolsep}{0.1cm}
\begin{center}
\begin{tabular}{lllllll|llllll}
%
\toprule
    Algorithm 
    &\multicolumn{6}{c}{ Recommended} & \multicolumn{6}{c}{Allocated} \\
\hline
    &   \multicolumn{1}{c}{1}
    &   \multicolumn{1}{c}{2}
    &   \multicolumn{1}{c}{3}
    &   \multicolumn{1}{c}{4}
    &   \multicolumn{1}{c}{5}
    &   \multicolumn{1}{c}{6}
    &   \multicolumn{1}{c}{1}
    &   \multicolumn{1}{c}{2}
    &   \multicolumn{1}{c}{3}
    &   \multicolumn{1}{c}{4}
    &   \multicolumn{1}{c}{5}
    &   \multicolumn{1}{c}{6} \\
\hline
%
%
\textbf{Sc. 1:} Tox prob \ & $\underline{0.30}$ & $0.45$ & $0.55$ & $0.60$ & $0.75$ & $0.80$ & $\underline{0.30}$ & $0.45$ & $0.55$ & $0.60$ & $0.75$ & $0.80$ \\
\hline
3 + 3 \hfill  [30.4] & \tblopt{35.2} & 21.6 & 7.7 & 4.0 & 1.0 & 0.1 & \hspace{0.15cm}- & \hspace{0.15cm}- &\hspace{0.15cm}- & \hspace{0.15cm}- & \hspace{0.15cm}- & \hspace{0.15cm}- \\
\hdashline
CRM &             \tblopt{77.2}  & 20.8 & 1.9  &  0.1  & 0.0 &  0.0  
                       &              \makecell{\tblopt{70.1} \\ (32.1)} &  \makecell{21.7 \\ (24.1)} &  \makecell{6.2 \\ (12.1)} &  \makecell{1.5 \\ (5.4)} &  \makecell{0.3 \\ (1.9)} &  \makecell{0.3 \\ (1.7)} \\
\hdashline
 $\mathrm{TS}$ &  \tblwinrec{\tblopt{78.9}} &  18.9 & 2.2 & 0.0 & 0.0 & 0.0 & \makecell{\tblopt{67.0} \\ (24.4)} &  \makecell{18.8 \\ (16.2)} &   \makecell{6.3 \\ (9.5)} &  \makecell{2.3 \\ (5.2)} &  \makecell{1.0 \\ (3.1)} &  \makecell{4.6 \\ (5.6)} \\
\hdashline
 $\mathrm{TS}(\epsilon)$ &  \tblwinrec{\tblopt{78.6}} &  19.5 & 1.8 & 0.1 & 0.0 & 0.0 &  \makecell{\tblwinrec{\tblopt{73.0}} \\ (31.0)} &  \makecell{20.7 \\ (24.3)} &  \makecell{5.2 \\ (10.8)} &  \makecell{0.7 \\ (3.6)} &  \makecell{0.1 \\ (1.2)} &  \makecell{0.3 \\ (1.6)} \\
\hdashline
 $\mathrm{TS}\_\mathrm{A}$ & \tblwinrec{\tblopt{79.8}}  & 18.2  & 1.7 &  0.2  & 0.1 &  0.0 
 &  \makecell{\tblwinrec{\tblopt{76.3}} \\ (24.1)} &  \makecell{19.7 \\ (18.7)} &   \makecell{3.5 \\ (8.7)} &  \makecell{0.5 \\ (3.2)} &  \makecell{0.1 \\ (1.0)} &  \makecell{0.0 \\ (0.3)} \\
\hdashline
$\mathrm{Independent \ TS}$ & \tblopt{37.6} & 27.3 & 16.9 & 13.2 & 3.2 & 1.9 &  \makecell{\tblopt{23.4}\\ (12.9)} &   \makecell{21.0 \\ (11.5)} &  \makecell{17.4 \\ (10.2)} &  \makecell{15.9 \\ (9.5)} &  \makecell{11.7 \\ (6.2)} &  \makecell{10.6 \\ (5.5)}\\
\midrule
\textbf{Sc. 2:} Tox prob \ & $0.05$ & $0.12$ & $0.15$ & $\underline{0.30}$ & $0.45$ & $0.50$ &  $0.05$ & $0.12$ & $0.15$ & $\underline{0.30}$ & $0.45$ & $0.50$ \\
\hline
3 + 3 \hfill [0.5] \ & 4.9 & 6.2 & 23.0 & \underline{28.6} & 18.8 & 17.9 & \hspace{0.15cm} - & \hspace{0.15cm}- &\hspace{0.15cm}- & \hspace{0.15cm}- & \hspace{0.15cm}- & \hspace{0.15cm}- \\
\hdashline
CRM &  0.2  &  1.2 & 17.1 &  \tblopt{{53.9}} & 21.7 & 5.9 
&   \makecell{10.3 \\ (6.4)} &  \makecell{10.7 \\ (11.1)} &  \makecell{20.6 \\ (19.6)} &  \makecell{\tblopt{29.9} \\ (21.0)} &  \makecell{15.9 \\ (16.3)} &  \makecell{12.7 \\ (16.6)} \\
\hdashline
$\mathrm{TS}$ & 0.0 & 1.2 & 17.8 & \tblopt{47.2} & 25.9 & 8.0 &  \makecell{13.6 \\ (8.8)} &  \makecell{14.6 \\ (10.6)} &  \makecell{18.2 \\ (13.0)} &  \makecell{\tblopt{20.3} \\ (13.8)} &  \makecell{12.5 \\ (11.2)} &  \makecell{20.7 \\ (15.2)} \\
 \hdashline
 $\mathrm{TS}(\epsilon)$ & 0.2 & 1.5 & 14.9 & \tblopt{51.5} & 24.2 & 7.6 &  \makecell{10.4 \\ (6.8)} &  \makecell{11.1 \\ (11.3)} &  \makecell{22.3 \\ (19.6)} &  \makecell{\tblwinrec{\tblopt{30.2}} \\ (21.2)} &  \makecell{13.0 \\ (15.3)} &  \makecell{13.0 \\ (16.0)} \\
  \hdashline
 $\mathrm{TS}\_\mathrm{A}$ &  0.0 & 1.8&  14.9 &  \tblopt{44.3} & 24.9  & 14.1
  &  \makecell{15.3 \\ (11.6)} &  \makecell{19.5 \\ (14.6)} &  \makecell{25.7 \\ (15.5)} &  \makecell{\tblopt{23.9} \\ (17.0)} &  \makecell{10.2 \\ (12.3)} &   \makecell{5.5 \\ (11.1)} \\
\hdashline
 $\mathrm{Independent \ TS}$ & 17.7 &  17.2 & 19.2 & \tblopt{20.2} & 14.7 & 11.0 
& \makecell{16.0 \\ (8.8)} &  \makecell{18.6 \\ (7.5)} &  \makecell{18.7 \\ (8.0)} &  \makecell{\tblopt{17.6} \\ (8.9)} &  \makecell{15.0 \\ (8.3)} &   \makecell{14.2 \\ (8.0)} \\
\hline
\textbf{Sc. 3:} Tox prob \ & $0.01$ & $0.03$ & $0.07$ & $0.11$ & $0.15$ & $\underline{0.30}$ & $0.01$ & $0.03$ & $0.07$ & $0.11$ & $0.15$ & $\underline{0.30}$ \\
\hline
3 + 3 \hfill [0.0] \ & 0.3 & 1.8 & 3.6 & 6.2 & 20.8 & \tblopt{67.2}& 
\hspace{0.15cm} - & \hspace{0.15cm}- &\hspace{0.15cm}- & \hspace{0.15cm}- & \hspace{0.15cm}- & \hspace{0.15cm}- \\
\hdashline
                        CRM &  9.6 &  0.0 & 0.1  &  1.4 & 14.8  &  \tblopt{74.1} 
                       &  \makecell{14.0 \\ (15.1)} &   \makecell{8.2 \\ (1.8)} &   \makecell{8.9 \\ (4.5)} &    \makecell{8.7 \\ (8.2)} &  \makecell{14.8 \\ (14.5)} &              \makecell{\tblopt{45.4} \\ (21.4)} \\
\hdashline
 $\mathrm{TS}$ & 2.9 & 0.0 & 0.1 & 1.8 & 14.8 & \tblwinrec{\tblopt{80.2}} &  \makecell{11.8 \\ (9.8)} &  \makecell{9.2 \\ (3.6)} &  \makecell{10.1 \\ (6.2)} &  \makecell{11.7 \\ (8.8)} &  \makecell{14.1 \\ (10.5)} &  \makecell{\tblopt{43.2} \\ (16.3)} \\
 \hdashline
 $\mathrm{TS}(\epsilon)$ & 2.9 & 0.1 & 0.1 & 1.5 & 15.8 & \tblwinrec{\tblopt{79.8}} &  \makecell{11.0 \\ (8.9)} &  \makecell{8.4 \\ (2.5)} &   \makecell{9.0 \\ (4.7)} &  \makecell{10.5 \\ (9.2)} &  \makecell{15.3 \\ (13.9)} &  \makecell{\tblwinrec{\tblopt{45.8}} \\ (19.1)} \\
\hdashline
 $\mathrm{TS}\_\mathrm{A}$ &  2.5 & 0.0& 0.1  & 1.7 & 14.3& \tblwinrec{\tblopt{81.5}}
 &   \makecell{11.7 \\ (9.4)} &  \makecell{10.6 \\ (6.2)} &  \makecell{13.6 \\ (9.7)} &  \makecell{15.9 \\ (10.9)} &  \makecell{16.0 \\ (10.6)} &              \makecell{\tblopt{32.1} \\ (19.0)} \\
\hdashline
 $\mathrm{Independent \ TS}$ & 18.8 & 10.0 & 14.4 & 19.4 & 18.6 & \tblopt{19.0}
 & \makecell{15.3 \\ (8.2)} &  \makecell{16.3 \\ (5.4)} &  \makecell{16.8 \\ (6.3)} &  \makecell{17.6 \\ (7.1)} &  \makecell{17.6 \\ (7.5)} &              \makecell{\tblopt{16.4} \\ (8.2)} \\
\bottomrule
\end{tabular}
\end{center}
\end{table*}

\begin{table*}[p]
\caption{Results for MTD identification (part 2/3)}\label{tbl-tox-b}
\renewcommand{\tabcolsep}{0.1cm}
\begin{center}
\begin{tabular}{lllllll|llllll}
%
\toprule
    Algorithm 
    &\multicolumn{6}{c}{ Recommended} & \multicolumn{6}{c}{Allocated} \\
\hline
    &   \multicolumn{1}{c}{1}
    &   \multicolumn{1}{c}{2}
    &   \multicolumn{1}{c}{3}
    &   \multicolumn{1}{c}{4}
    &   \multicolumn{1}{c}{5}
    &   \multicolumn{1}{c}{6}
    &   \multicolumn{1}{c}{1}
    &   \multicolumn{1}{c}{2}
    &   \multicolumn{1}{c}{3}
    &   \multicolumn{1}{c}{4}
    &   \multicolumn{1}{c}{5}
    &   \multicolumn{1}{c}{6} \\
\hline
\textbf{Sc. 4:} Tox prob \ &   $0.10$ & $0.20$ & $\underline{0.30}$ & $0.40$ & $0.47$ & $0.53$ &  $0.10$ & $0.20$ & $\underline{0.30}$ & $0.40$ & $0.47$ & $0.53$ \\
\hline
3 + 3 \hfill [4.7] \ & 12.5 & 20.5 & \tblopt{23.0} & 18.8 & 11.6 & 9.0 
&\hspace{0.15cm} - & \hspace{0.15cm}- &\hspace{0.15cm}- & \hspace{0.15cm}- & \hspace{0.15cm}- & \hspace{0.15cm}- \\
\hdashline
CRM & 1.2  &  22.0  &  \tblopt{42.2}& 25.7 &  6.9 & 2.1 
&  \makecell{14.6 \\ (13.7)} &  \makecell{23.1 \\ (23.7)} &              \makecell{\tblopt{30.6} \\ (24.0)} &  \makecell{18.0 \\ (19.4)} &  \makecell{7.4 \\ (12.3)} &   \makecell{6.3 \\ (12.2)} \\
\hdashline
$\mathrm{TS}$ & 1.2 & 19.6 & \tblopt{40.1} & 28.2 & 8.1 & 2.8 &  \makecell{21.4 \\ (15.1)} &  \makecell{21.6 \\ (15.4)} &  \makecell{\tblopt{20.8} \\ (15.0)} &  \makecell{13.7 \\ (12.3)} &  \makecell{6.8 \\ (8.7)} &  \makecell{15.7 \\ (13.6)} \\
\hdashline
 $\mathrm{TS}(\epsilon)$ & 2.1 & 19.9 & \tblwinrec{\tblopt{44.1}} & 24.9 & 7.0 & 1.8 &  \makecell{15.5 \\ (15.8)} &  \makecell{25.3 \\ (24.5)} &  \makecell{\tblopt{31.8} \\ (24.5)} &  \makecell{16.2 \\ (19.2)} &  \makecell{5.1 \\ (9.8)} &   \makecell{6.1 \\ (10.8)} \\
\hdashline
 $\mathrm{TS}\_\mathrm{A}$ &  1.4 &  20.6  & \tblwinrec{\tblopt{42.3}} & 22.2 &  9.0  &  4.5 
 &  \makecell{25.1 \\ (19.3)} &  \makecell{31.0 \\ (18.4)} &              \makecell{\tblopt{27.4} \\ (18.8)} &  \makecell{11.9 \\ (14.9)} &   \makecell{3.2 \\ (7.9)} &    \makecell{1.3 \\ (5.6)} \\
\hdashline
 $\mathrm{Independent \ TS}$ & 17.8 & 22.2 & \tblopt{22.6} & 15.9 & 12.6 & 9.0 
 & \makecell{16.6 \\ (9.3)} &  \makecell{19.4 \\ (8.7)} & \makecell{\tblopt{18.7} \\ (9.3)} &  \makecell{16.5 \\ (8.8)} &   \makecell{15.3 \\ (8.6)} &    \makecell{13.5 \\ (7.7)}\\ 
\midrule
\textbf{Sc. 5:} Tox prob \ &   $0.10$ & $\underline{0.25}$ & $0.40$ & $0.50$ & $0.65$ & $0.75$  &  $0.10$ & $\underline{0.25}$ & $0.40$ & $0.50$ & $0.65$ & $0.75$  \\
\hline
3 + 3 \hfill [3.1] \ & 20.6 & \tblopt{30.8} & 24.2 & 15.3 & 5.1 & 0.8
& \hspace{0.15cm} - & \hspace{0.15cm}- &\hspace{0.15cm}- & \hspace{0.15cm}- & \hspace{0.15cm}- & \hspace{0.15cm}- \\
\hdashline
CRM &  4.8 &  \tblopt{49.7}&  39.0 &  6.5  & 0.1 &  0.0
&  \makecell{17.8 \\ (18.2)} &              \makecell{\tblopt{38.3} \\ (27.4)} &  \makecell{30.9 \\ (23.9)} &  \makecell{9.0 \\ (14.8)} &  \makecell{2.4 \\ (5.5)} &  \makecell{1.7 \\ (4.0)} \\
\hdashline
 $\mathrm{TS}$ & 4.3 & \tblwinrec{\tblopt{50.7}} & 39.4 & 5.4 & 0.1 & 0.1 &  \makecell{26.3 \\ (17.6)} &  \makecell{\tblopt{31.2} \\ (17.5)} &  \makecell{22.3 \\ (16.0)} &  \makecell{8.8 \\ (11.4)} &  \makecell{3.2 \\ (5.4)} &  \makecell{8.2 \\ (7.2)} \\
 \hdashline
 $\mathrm{TS}(\epsilon)$ & 4.8 & \tblwinrec{\tblopt{52.2}} & 36.5 & 6.2 & 0.2 & 0.0 &  \makecell{18.8 \\ (19.3)} &  \makecell{\tblwinrec{\tblopt{41.2}} \\ (27.1)} &  \makecell{29.7 \\ (24.4)} &  \makecell{7.3 \\ (13.7)} &  \makecell{1.4 \\ (4.2)} &  \makecell{1.6 \\ (3.9)} \\
 \hdashline
  $\mathrm{TS}\_\mathrm{A}$ &  3.0  &  \tblwinrec{\tblopt{50.8}}& 36.4  & 7.0  &  1.6 &  1.1 
 &  \makecell{29.6 \\ (20.0)} &  \makecell{\tblwinrec{\tblopt{40.1}} \\ (18.8)} &  \makecell{23.4 \\ (18.5)} &  \makecell{6.1 \\ (11.0)} &  \makecell{0.8 \\ (3.2)} &  \makecell{0.1 \\ (1.1)} \\
\hdashline
 $\mathrm{Independent \ TS}$ & 24.3 & \tblopt{32.6} &  21.4 & 14.6 & 5.4 & 1.6
 & \makecell{19.4 \\ (10.5)} &  \makecell{\tblopt{22.6} \\ (10.8)} &  \makecell{19.1 \\ (10.0)} &  \makecell{16.0 \\ (9.1)} &  \makecell{12.5 \\ (7.0)} &  \makecell{10.4 \\ (5.5)}\\
\hline
\textbf{Sc. 6:} Tox prob \ &    $0.08$ & $0.12$ & $0.18$ & $\underline{0.25}$ & $\underline{0.33}$ & $0.39$ &  $0.08$ & $0.12$ & $0.18$ & $\underline{0.25}$ & $\underline{0.33}$ & $0.39$ \\
\hline
3 + 3 \hfill [2.1] \ & 5.1 & 9.6 & 15.3 & \tblopt{19.3} & \tblopt{18.5} &  30.2
& \hspace{0.15cm} - & \hspace{0.15cm}- &\hspace{0.15cm}- & \hspace{0.15cm}- & \hspace{0.15cm}- & \hspace{0.15cm}- \\
\hdashline
CRM &  0.3  & 1.2  &  10.6  &  \tblopt{29.1} & \tblopt{31.2}  & 27.5  
&   \makecell{11.7 \\ (8.4)} &  \makecell{10.7 \\ (11.4)} &  \makecell{16.2 \\ (17.1)} &  \makecell{\tblopt{19.5} \\ (19.3)} &  \makecell{\tblopt{18.2} \\ (18.0)} &  \makecell{23.7 \\ (23.9)} \\
\hdashline
 $\mathrm{TS}$ & 0.3 & 1.4 & 10.6 & \tblopt{27.0} & \tblopt{29.9} & 30.9 &  \makecell{14.9 \\ (10.7)} &   \makecell{13.2 \\ (9.9)} &  \makecell{15.3 \\ (12.0)} &  \makecell{\tblopt{15.1} \\ (12.3)} &  \makecell{\tblopt{12.2} \\ (11.3)} &  \makecell{29.2 \\ (18.7)} \\
 \hdashline
 $\mathrm{TS}(\epsilon)$ & 0.1 & 1.7 & 11.5 & \tblopt{28.3} & \tblopt{30.4} & 28.0 &   \makecell{12.0 \\ (8.9)} &  \makecell{11.9 \\ (12.5)} &  \makecell{19.2 \\ (18.6)} &  \makecell{\tblwinrec{\tblopt{20.3}} \\ (19.2)} &  \makecell{\tblopt{13.5} \\ (15.4)} &  \makecell{23.0 \\ (22.7)} \\
\hdashline
 $\mathrm{TS}\_\mathrm{A}$ &  0.1 &  1.9  &  12.0 &  \tblopt{28.5}  & \tblopt{26.5}&  31.0 
 &  \makecell{17.5 \\ (14.4)} &  \makecell{21.1 \\ (15.0)} &  \makecell{24.7 \\ (15.9)} &  \makecell{\tblopt{19.3} \\ (15.9)} &   \makecell{\tblopt{8.9} \\ (11.5)} &   \makecell{8.5 \\ (15.0)} \\
\hdashline
 $\mathrm{Independent \ TS}$ & 13.6 & 15.6 & 19.1 & \tblopt{19.4} & \tblopt{16.8} & 15.4
& \makecell{14.7 \\ (8.4)} &  \makecell{17.7 \\ (7.4)} &  \makecell{18.0 \\ (8.0)} &  \makecell{\tblopt{17.5} \\ (8.5)} &   \makecell{\tblopt{16.4} \\ (8.5)} &   \makecell{15.7 \\ (8.4)}\\
 \bottomrule
\end{tabular}
\end{center}
\end{table*}
\begin{table*}[p]
\caption{Results for MTD identification (part 3/3)}\label{tbl-tox-c}
\renewcommand{\tabcolsep}{0.1cm}
\begin{center}
\begin{tabular}{lllllll|llllll}
%
\toprule
    Algorithm 
    &\multicolumn{6}{c}{ Recommended} & \multicolumn{6}{c}{Allocated} \\
\hline
    &   \multicolumn{1}{c}{1}
    &   \multicolumn{1}{c}{2}
    &   \multicolumn{1}{c}{3}
    &   \multicolumn{1}{c}{4}
    &   \multicolumn{1}{c}{5}
    &   \multicolumn{1}{c}{6}
    &   \multicolumn{1}{c}{1}
    &   \multicolumn{1}{c}{2}
    &   \multicolumn{1}{c}{3}
    &   \multicolumn{1}{c}{4}
    &   \multicolumn{1}{c}{5}
    &   \multicolumn{1}{c}{6} \\
\hline
 \textbf{Sc. 7:} Tox prob \ & $0.15$ & $\underline{0.30}$ & $0.45$ & $0.50$ & $0.60$ & $0.70$ & $0.15$ & $\underline{0.30}$ & $0.45$ & $0.50$ & $0.60$ & $0.70$ \\
\hline
3 + 3 \hfill [7.7] \ & 24.7 & \tblopt{32.8} & 18.0 & 10.2 & 4.9 & 1.8 
& \hspace{0.15cm} - & \hspace{0.15cm}- &\hspace{0.15cm}- & \hspace{0.15cm}- & \hspace{0.15cm}- & \hspace{0.15cm}- \\
\hdashline
CRM &  16.9  &  \tblopt{59.4} &  20.4  & 3.0  & 0.2  &  0.2 
&  \makecell{27.7 \\ (27.2)} &              \makecell{\tblopt{40.8} \\ (27.1)} &  \makecell{22.4 \\ (22.1)} &  \makecell{6.0 \\ (11.6)} &  \makecell{1.8 \\ (5.5)} &  \makecell{1.4 \\ (4.2)} \\
\hdashline
$\mathrm{TS}$ &  14.5 &  \tblopt{55.7} &  25.6 &  3.9 &  0.1 &  0.1 &  \makecell{34.9 \\ (21.8)} &  \makecell{\tblopt{29.5} \\ (17.0)} &  \makecell{17.5 \\ (14.8)} &   \makecell{7.0 \\ (9.8)} &  \makecell{2.9 \\ (5.4)} &  \makecell{8.2 \\ (8.0)} \\
\hdashline
 $\mathrm{TS}(\epsilon)$ &  15.0 &  \tblopt{58.0} &  23.2 &  3.5 &  0.2 &  0.1 &  \makecell{28.8 \\ (27.5)} &  \makecell{\tblwinrec{\tblopt{43.3}} \\ (27.0)} &  \makecell{20.8 \\ (21.9)} &  \makecell{4.7 \\ (11.0)} &  \makecell{1.0 \\ (4.0)} &  \makecell{1.5 \\ (4.0)} \\
\hdashline
 $\mathrm{TS}\_\mathrm{A}$ &  13.7 &  \tblwinrec{\tblopt{59.5}} & 21.5  & 3.7 & 0.9  &  0.8
&  \makecell{41.7 \\ (24.8)} &              \makecell{\tblopt{39.3} \\ (18.9)} &  \makecell{15.5 \\ (16.8)} &   \makecell{3.1 \\ (7.9)} &  \makecell{0.4 \\ (2.7)} &  \makecell{0.1 \\ (1.2)} \\
\hdashline
 $\mathrm{Independent \ TS}$ & 25.4 & \tblopt{33.1} & 16.8 & 13.7 & 7.6 & 3.4 
 & \makecell{19.2 \\ (11.0)} & \makecell{\tblopt{22.5} \\ (11.1)} &  \makecell{17.5 \\ (9.9)} &   \makecell{16.3 \\ (9.4)} &  \makecell{13.3 \\ (7.7)} &  \makecell{11.2 \\ (6.3)}\\
\midrule
\textbf{Sc. 8:} Tox prob \ & $0.10$ & $0.15$ & $\underline{0.30}$ & $0.45$ & $0.60$ & $0.75$ & $0.10$ & $0.15$ & $\underline{0.30}$ & $0.45$ & $0.60$ & $0.75$ \\
\hline
3 + 3 \hfill [3.1] \ & 6.8 & 24.1 & \tblopt{30.8} & 22.4 & 11.0 & 1.8
& \hspace{0.15cm} - & \hspace{0.15cm}- &\hspace{0.15cm}- & \hspace{0.15cm}- & \hspace{0.15cm}- & \hspace{0.15cm}- \\
\hdashline
CRM &  1.1  & 15.1  & \tblopt{60.6}  & 21.6 &  1.6 &  0.1  
&  \makecell{13.5 \\ (12.4)} &  \makecell{20.4 \\ (20.8)} &  \makecell{\tblopt{39.6} \\ (24.5)} &  \makecell{18.4 \\ (20.2)} &  \makecell{4.9 \\ (9.3)} &  \makecell{3.1 \\ (5.5)} \\
\hdashline
$\mathrm{TS}$ &  0.9 &  21.0 &  \tblopt{58.5} &  18.5 &  1.0 &  0.1 &  \makecell{20.4 \\ (15.4)} &  \makecell{23.8 \\ (15.0)} &  \makecell{\tblopt{27.4} \\ (16.3)} &  \makecell{13.8 \\ (13.5)} &  \makecell{4.9 \\ (7.3)} &  \makecell{9.8 \\ (7.6)} \\
\hdashline
 $\mathrm{TS}(\epsilon)$ &  0.8 &  17.0 &  \tblopt{59.4} &  20.2 &  2.0 &  0.7 &  \makecell{14.4 \\ (14.0)} &  \makecell{23.4 \\ (21.8)} &  \makecell{\tblwinrec{\tblopt{39.9}} \\ (24.3)} &  \makecell{16.3 \\ (19.8)} &  \makecell{2.6 \\ (6.1)} &  \makecell{3.4 \\ (5.4)} \\
\hdashline
 $\mathrm{TS}\_\mathrm{A}$ &  0.3  &  14.5 & \tblopt{51.9}  & 24.0 & 5.4&  3.9 
 &  \makecell{22.4 \\ (17.5)} &  \makecell{30.1 \\ (17.6)} &  \makecell{\tblopt{31.7} \\ (18.3)} &  \makecell{13.0 \\ (14.7)} &  \makecell{2.3 \\ (6.0)} &  \makecell{0.5 \\ (2.3)} \\
\hdashline
 $\mathrm{Independent \ TS}$ & 22.4 & 24.4 & \tblopt{26.8} & 17.1 & 7.4 & 2.0 
 &  \makecell{18.3 \\ (9.9)} &  \makecell{21.5 \\ (9.1)} &  \makecell{\tblopt{20.5} \\ (9.8)} &  \makecell{16.9 \\ (9.4)} &  \makecell{13.0 \\ (7.4)} &  \makecell{9.9 \\ (5.3)} \\
\hline
\textbf{Sc. 9:} Tox prob \ & $0.01$ & $0.05$ & $0.08$ & $0.15$ & $\underline{0.30}$ & $0.45$ & $0.01$ & $0.05$ & $0.08$ & $0.15$ & $\underline{0.30}$ & $0.45$\\
\hline
3 + 3 \hfill [0.1] \ &  0.8 & 2.1 & 8.0 & 23.8 & \tblopt{30.0} & 35.2
& \hspace{0.15cm} - & \hspace{0.15cm}- &\hspace{0.15cm}- & \hspace{0.15cm}- & \hspace{0.15cm}- & \hspace{0.15cm}- \\
\hdashline
CRM &  1.9  &  0.1  &  0.4 &  16.1 & \tblopt{54.1}& 27.4
&   \makecell{9.8 \\ (7.3)} &   \makecell{8.5 \\ (3.5)} &   \makecell{10.0 \\ (7.6)} &  \makecell{17.0 \\ (16.4)} &  \makecell{\tblopt{28.9} \\ (18.9)} &  \makecell{25.8 \\ (20.8)} \\
\hdashline
$\mathrm{TS}$ &  0.5 &  0.0 &  0.5 &  17.1 &  \tblopt{50.8} &  31.1 &  \makecell{10.3 \\ (6.0)} &  \makecell{10.0 \\ (4.8)} &  \makecell{12.0 \\ (7.9)} &  \makecell{18.3 \\ (12.1)} &  \makecell{\tblopt{20.0} \\ (12.7)} &  \makecell{29.4 \\ (16.0)} \\
\hdashline
 $\mathrm{TS}(\epsilon)$ &  0.7 &  0.1 &  0.4 &  15.2 &  \tblwinrec{\tblopt{55.9}} &  27.9 &   \makecell{9.3 \\ (5.3)} &   \makecell{8.4 \\ (2.5)} &  \makecell{10.6 \\ (7.3)} &  \makecell{20.2 \\ (17.2)} &  \makecell{\tblopt{26.3} \\ (18.3)} &  \makecell{25.2 \\ (19.9)} \\
\hdashline
 $\mathrm{TS}\_\mathrm{A}$ &   0.3  &  0.0  & 0.5 & 13.2 &  \tblopt{46.7} &  39.2 
 &  \makecell{10.4 \\ (5.8)} &  \makecell{12.3 \\ (8.4)} &  \makecell{16.5 \\ (11.5)} &  \makecell{22.9 \\ (14.2)} &  \makecell{\tblopt{19.9} \\ (13.3)} &  \makecell{18.1 \\ (17.1)} \\
\hdashline
 $\mathrm{Independent \ TS}$ & 18.8 & 11.6 & 14.8 & 19.0 & \tblopt{21.0} & 14.8
 & \makecell{15.4 \\ (8.4)} &  \makecell{17.1 \\ (6.2)} &  \makecell{17.5 \\ (6.8)} &  \makecell{18.1 \\ (7.6)} &  \makecell{\tblopt{17.1} \\ (8.4)} &  \makecell{14.9 \\ (7.9)} \\ 
\bottomrule
\end{tabular}
\end{center}
\end{table*}

\begin{table}[p]
\caption{Results for MED identification (part 1/3).}
\label{tbl-eff}
\renewcommand{\tabcolsep}{0.3em}
\centering
\begin{tabular}{cccccccc|cccccc}
\hline
    Algorithm &  E-Stop 
    &\multicolumn{6}{c}{ Recommended} & \multicolumn{6}{c}{Allocated} \\
\hline
    &
    &   \multicolumn{1}{c}{1}
    &   \multicolumn{1}{c}{2}
    &   \multicolumn{1}{c}{3}
    &   \multicolumn{1}{c}{4}
    &   \multicolumn{1}{c}{5}
    &   \multicolumn{1}{c}{6}
    &   \multicolumn{1}{c}{1}
    &   \multicolumn{1}{c}{2}
    &   \multicolumn{1}{c}{3}
    &   \multicolumn{1}{c}{4}
    &   \multicolumn{1}{c}{5}
    &   \multicolumn{1}{c}{6} \\
\hline
\multicolumn{2}{l}{\textbf{Sc. 1:} Tox prob} & 0.01  & 0.05 & \underline{0.15} & 0.2 & 0.45 & 0.6  &  0.01  & 0.05 & \underline{0.15} & 0.2 & \dash{0.45} & \dash{0.6} \\ 
\multicolumn{2}{l}{\textbf{Sc. 1:} Eff prob}  & 0.1 & 0.35 & \underline{0.6} & 0.6 & 0.6 & 0.6 &  0.1 & 0.35 & \underline{0.6} & 0.6 & 0.6 & 0.6 \\ 
\hline
 $\mathrm{MTA}$-$\mathrm{RA}$ & 0.4 &  0.4 &  7.0 &              \tblopt{54.9} &  29.1 &  7.4 &  0.8 &   \makecell{7.1 \\ (3.8)} &  \makecell{14.2 \\ (13.9)} &  \makecell{\tblopt{37.9} \\ (24.4)} &  \makecell{24.9 \\ (18.8)} &  \makecell{\dash{12.9} \\ (13.6)} &  \makecell{\dash{2.5} \\ (4.9)} \\
 \hdashline
               $\mathrm{TS}$ & 0.9 &  0.1 &  9.7 &  \tblwinrec{\tblopt{57.6}} &  27.0 &  4.2 &  0.4 &  \makecell{10.6 \\ (5.7)} &  \makecell{18.4 \\ (11.0)} &  \makecell{\tblopt{31.9} \\ (14.4)} &  \makecell{23.8 \\ (13.2)} &          \makecell{10.0 \\ (8.0)} &         \makecell{4.4 \\ (4.5)} \\
\hdashline
    $\mathrm{TS}\_\mathrm{A}$ & 0.9 &  0.3 &  9.6 &  \tblwinrec{\tblopt{59.4}} &  26.1 &  3.5 &  0.2 &  \makecell{10.7 \\ (5.4)} &  \makecell{20.7 \\ (12.9)} &  \makecell{\tblopt{35.7} \\ (14.9)} &  \makecell{23.9 \\ (14.1)} &    \makecell{\dash{7.3} \\ (8.1)} &  \makecell{\dash{0.9} \\ (2.7)} \\
\hline
\multicolumn{2}{l}{\textbf{Sc. 2:} Tox prob}& 0.005  & 0.01 & 0.02 & 0.05 & \underline{0.1} & 0.15 & 0.005  & 0.01 & 0.02 & 0.05 & \underline{0.1} & 0.15\\  
\multicolumn{2}{l}{\textbf{Sc. 2:} Eff prob} &  0.001  & 0.1 & 0.3 & 0.5 & \underline{0.8} & 0.8 &  0.001  & 0.1 & 0.3 & 0.5 & \underline{0.8} & 0.8\\
 \hline
 $\mathrm{MTA}$-$\mathrm{RA}$ & 1.9 &  0.0 &  0.1 &  1.6 &  5.1 &              \tblopt{55.0} &  36.2 &  \makecell{5.2 \\ (1.7)} &  \makecell{5.6 \\ (3.1)} &   \makecell{7.5 \\ (8.5)} &  \makecell{11.4 \\ (13.6)} &  \makecell{\tblopt{36.7} \\ (25.8)} &  \makecell{31.7 \\ (26.9)} \\
 \hdashline
               $\mathrm{TS}$ &  0.8 &  0.0 &  0.0 &   0.5 &  4.7 &  \tblwinrec{\tblopt{56.6}} &  37.5 &  \makecell{5.9 \\ (2.4)} &  \makecell{6.6 \\ (3.4)} &   \makecell{9.3 \\ (6.0)} &   \makecell{16.9 \\ (9.7)} &  \makecell{\tblopt{32.5} \\ (13.3)} &  \makecell{28.1 \\ (14.4)} \\
\hdashline
    $\mathrm{TS}\_\mathrm{A}$ & 2.2 &  0.0 &  0.1 &  1.6 &  5.0 &  \tblwinrec{\tblopt{55.9}} &  35.2 &  \makecell{5.9 \\ (2.3)} &  \makecell{6.8 \\ (3.8)} &  \makecell{10.9 \\ (8.7)} &  \makecell{17.9 \\ (10.8)} &  \makecell{\tblopt{31.8} \\ (14.3)} &  \makecell{24.5 \\ (15.5)} \\
\hline
\multicolumn{2}{l}{\textbf{Sc. 3:} Tox prob}  & \underline{0.01}  & 0.05 & 0.1 & 0.25 & 0.5 & 0.7 & \underline{0.01}  & 0.05 & 0.1 & 0.25 & \dash{0.5} & \dash{0.7}\\ 
\multicolumn{2}{l}{\textbf{Sc. 3:} Eff prob}  & \underline{0.4}  & 0.4 & 0.4 & 0.4 & 0.4 & 0.4 & \underline{0.4}  & 0.4 & 0.4 & 0.4 & 0.4 & 0.4 \\ 
\hline
 $\mathrm{MTA}$-$\mathrm{RA}$ & 0.4 &              \tblopt{51.5} &  26.4 &  12.5 &  6.8 &  2.2 &  0.2 &  \makecell{\tblopt{38.2} \\ (25.2)} &  \makecell{24.8 \\ (17.9)} &  \makecell{16.6 \\ (13.9)} &  \makecell{12.9 \\ (12.3)} &  \makecell{\dash{6.1} \\ (8.4)} &  \makecell{\dash{0.9} \\ (2.7)} \\
 \hdashline
               $\mathrm{TS}$ & 0.1 &  \tblwinrec{\tblopt{53.9}} &  24.8 &  12.2 &  7.8 &  1.1 &  0.1 &  \makecell{\tblopt{24.1} \\ (11.4)} &   \makecell{22.7 \\ (9.8)} &  \makecell{23.8 \\ (10.9)} &  \makecell{19.0 \\ (10.6)} &         \makecell{7.2 \\ (6.1)} &         \makecell{3.1 \\ (3.6)} \\
\hdashline
    $\mathrm{TS}\_\mathrm{A}$ & 0.5 &  \tblwinrec{\tblopt{53.8}} &  26.4 &  10.4 &  8.2 &   0.7 &  0.1 &  \makecell{\tblopt{26.6} \\ (13.3)} &  \makecell{25.1 \\ (11.4)} &  \makecell{24.8 \\ (11.4)} &  \makecell{17.7 \\ (11.9)} &  \makecell{\dash{4.8} \\ (6.6)} &  \makecell{\dash{0.5} \\ (2.0)} \\
\hline
\multicolumn{2}{l}{\textbf{Sc. 4:} Tox prob} & 0.01  & 0.02 & \underline{0.05} & 0.1 & 0.2 & 0.3 & 0.01  & 0.02 & \underline{0.05} & 0.1 & 0.2 & 0.3 \\ 
\multicolumn{2}{l}{\textbf{Sc. 4:} Eff prob}  & 0.25  & 0.45 & \underline{0.65} & 0.65 & 0.65 & 0.65 & 0.25  & 0.45 & \underline{0.65} & 0.65 & 0.65 & 0.65 \\
\hline
 $\mathrm{MTA}$-$\mathrm{RA}$ & 0.1 &  1.8 &  13.2 &              \tblopt{49.0} &  21.7 &   8.5 &  5.7 &   \makecell{9.5 \\ (7.9)} &  \makecell{17.7 \\ (15.9)} &  \makecell{\tblopt{31.6} \\ (21.5)} &  \makecell{20.6 \\ (15.6)} &  \makecell{13.9 \\ (12.6)} &  \makecell{6.6 \\ (10.2)} \\
 \hdashline
               $\mathrm{TS}$ & 0.1 &  1.8 &  15.7 &              \tblopt{45.8} &  18.1 &  10.8 &  7.8 &  \makecell{12.1 \\ (6.8)} &   \makecell{16.8 \\ (8.9)} &  \makecell{\tblopt{23.1} \\ (11.0)} &  \makecell{21.6 \\ (10.2)} &   \makecell{16.5 \\ (9.3)} &   \makecell{9.8 \\ (7.6)} \\
\hdashline
    $\mathrm{TS}\_\mathrm{A}$ & 0.2 &  2.4 &  15.0 &  \tblwinrec{\tblopt{49.1}} &  20.2 &   9.8 &  3.2 &  \makecell{13.2 \\ (8.0)} &  \makecell{19.3 \\ (10.8)} &  \makecell{\tblopt{25.5} \\ (12.3)} &  \makecell{21.9 \\ (10.8)} &  \makecell{14.1 \\ (10.7)} &   \makecell{5.8 \\ (7.9)} \\
\hline
\multicolumn{2}{l}{\textbf{Sc. 5:} Tox prob} & 0.1  & 0.2 & \underline{0.25} & 0.4 & 0.5 & 0.6 & 0.1  & 0.2 & \underline{0.25} & \dash{0.4} & \dash{0.5} & \dash{0.6} \\ 
\multicolumn{2}{l}{\textbf{Sc. 5:} Eff prob} &   0.3  & 0.4 & \underline{0.5} & 0.7 & 0.7 & 0.7 & 0.3  & 0.4 & \underline{0.5} & 0.7 & 0.7 & 0.7 \\ 
\hline
 $\mathrm{MTA}$-$\mathrm{RA}$ & 1.4 &   9.0 &  13.2 &              \tblopt{25.9} &  40.6 &  8.3 &  1.5 &  \makecell{15.5 \\ (16.7)} &  \makecell{19.1 \\ (17.0)} &  \makecell{\tblopt{24.9} \\ (17.7)} &  \makecell{\dash{26.7} \\ (19.5)} &  \makecell{\dash{9.9} \\ (11.0)} &  \makecell{\dash{2.4} \\ (5.0)} \\
 \hdashline
               $\mathrm{TS}$ & 5.8 &   8.3 &  24.4 &  \tblwinrec{\tblopt{40.0}} &  18.9 &  2.4 &   0.3 &  \makecell{20.8 \\ (15.8)} &  \makecell{27.3 \\ (15.8)} &  \makecell{\tblopt{24.4} \\ (14.8)} &         \makecell{13.0 \\ (11.2)} &          \makecell{5.5 \\ (6.3)} &         \makecell{3.3 \\ (4.3)} \\
\hdashline
    $\mathrm{TS}\_\mathrm{A}$ & 6.9 &  16.7 &  30.6 &  \tblwinrec{\tblopt{30.6}} &  14.4 &   0.8 &   0.0 &  \makecell{25.9 \\ (19.2)} &  \makecell{33.8 \\ (18.6)} &  \makecell{\tblopt{22.8} \\ (16.1)} &   \makecell{\dash{8.6} \\ (11.7)} &   \makecell{\dash{1.8} \\ (4.8)} &  \makecell{\dash{0.2} \\ (1.3)} \\
\hline
\end{tabular}
\end{table}       
\begin{table}[p]
\caption{Results for MED identification (part 2/3).}
\label{tbl-eff-b}
\renewcommand{\tabcolsep}{.3em}
\centering
\begin{tabular}{lccccccc|cccccc}
    Algorithm &  E-Stop 
    &\multicolumn{6}{c}{ Recommended} & \multicolumn{6}{c}{Allocated} \\
    \hline
    &
    &   \multicolumn{1}{c}{1}
    &   \multicolumn{1}{c}{2}
    &   \multicolumn{1}{c}{3}
    &   \multicolumn{1}{c}{4}
    &   \multicolumn{1}{c}{5}
    &   \multicolumn{1}{c}{6}
    &   \multicolumn{1}{c}{1}
    &   \multicolumn{1}{c}{2}
    &   \multicolumn{1}{c}{3}
    &   \multicolumn{1}{c}{4}
    &   \multicolumn{1}{c}{5}
    &   \multicolumn{1}{c}{6} \\
\hline
\multicolumn{2}{l}{\textbf{Sc. 6:} Tox prob} & 0.1  & 0.3 & \underline{0.35} & 0.4 & 0.5 & 0.6 & 0.1  & 0.3 & \underline{0.35} & \dash{0.4} & \dash{0.5} & \dash{0.6} \\ 
\multicolumn{2}{l}{\textbf{Sc. 6:} Eff prob} &   0.3  & 0.4 & \underline{0.5} & 0.7 & 0.7 & 0.7 &   0.3  & 0.4 & \underline{0.5} & 0.7 & 0.7 & 0.7 \\ 
\hline
 $\mathrm{MTA}$-$\mathrm{RA}$ & 4.2 &  11.2 &  24.3 &  \tblopt{24.6} &  28.9 &  5.4 &  1.3 &  \makecell{17.9 \\ (19.9)} &  \makecell{24.2 \\ (22.2)} &  \makecell{\tblopt{23.7} \\ (18.6)} &  \makecell{\dash{20.7} \\ (19.8)} &  \makecell{\dash{7.7} \\ (10.5)} &  \makecell{\dash{1.7} \\ (4.2)} \\
 \hdashline
               $\mathrm{TS}$ & 8.4 &  17.8 &  41.9 &  \tblopt{22.4} &   8.1 &  1.1 &   0.2 &  \makecell{29.6 \\ (21.3)} &  \makecell{30.4 \\ (17.0)} &  \makecell{\tblopt{16.9} \\ (13.4)} &           \makecell{8.2 \\ (9.5)} &          \makecell{4.0 \\ (5.7)} &         \makecell{2.4 \\ (3.8)} \\
\hdashline
    $\mathrm{TS}\_\mathrm{A}$ & 9.4 &  28.5 &  43.6 &  \tblopt{14.2} &   4.0 &   0.2 &   0.0 &  \makecell{34.5 \\ (24.0)} &  \makecell{37.2 \\ (20.2)} &  \makecell{\tblopt{14.3} \\ (14.5)} &    \makecell{\dash{3.9} \\ (8.3)} &   \makecell{\dash{0.6} \\ (2.7)} &  \makecell{\dash{0.1} \\ (0.9)} \\
\hline
\multicolumn{2}{l}{\textbf{Sc. 7:} Tox prob} & 0.03  & \underline{0.06} & 0.1 & 0.2 & 0.4 & 0.5 & 0.03  & \underline{0.06} & 0.1 & 0.2 & \dash{0.4} & \dash{0.5} \\ 
\multicolumn{2}{l}{\textbf{Sc. 7:} Eff prob} & 0.3  & \underline{0.5} & 0.52 & 0.54 & 0.55 & 0.55 &  0.3  & \underline{0.5} & 0.52 & 0.54 & 0.55 & 0.55 \\ 
\hline 
 $\mathrm{MTA}$-$\mathrm{RA}$ & 0.1 &   8.6 &              \tblopt{45.5} &  25.1 &  13.7 &  5.7 &  1.4 &  \makecell{16.1 \\ (14.6)} &  \makecell{\tblopt{31.5} \\ (21.8)} &  \makecell{22.8 \\ (17.0)} &  \makecell{17.0 \\ (14.0)} &  \makecell{\dash{9.9} \\ (10.9)} &  \makecell{\dash{2.5} \\ (6.1)} \\
 \hdashline
               $\mathrm{TS}$ & 0.7 &  10.3 &              \tblopt{43.7} &  22.1 &  16.3 &  5.7 &  1.2 &   \makecell{17.5 \\ (8.9)} &  \makecell{\tblopt{22.7} \\ (11.3)} &  \makecell{23.2 \\ (10.8)} &  \makecell{20.6 \\ (11.0)} &         \makecell{10.3 \\ (7.6)} &         \makecell{4.9 \\ (5.2)} \\
\hdashline
    $\mathrm{TS}\_\mathrm{A}$ & 0.4 &  11.3 &  \tblwinrec{\tblopt{47.9}} &  22.8 &  13.3 &  4.2 &   0.1 &  \makecell{19.8 \\ (11.2)} &  \makecell{\tblopt{26.9} \\ (13.3)} &  \makecell{26.1 \\ (11.6)} &  \makecell{19.0 \\ (12.6)} &   \makecell{\dash{6.6} \\ (8.0)} &  \makecell{\dash{1.2} \\ (3.5)} \\
\hline
\multicolumn{2}{l}{\textbf{Sc. 8:} Tox prob} & 0.02  & 0.07 & \underline{0.13} & 0.17 & 0.25 & 0.3 & 0.02  & 0.07 & \underline{0.13} & 0.17 & 0.25 & 0.3\\ 
\multicolumn{2}{l}{\textbf{Sc. 8:} Eff prob}  & 0.3  & 0.5 & \underline{0.7} & 0.73 & 0.76 & 0.77 & 0.3  & 0.5 & \underline{0.7} & 0.73 & 0.76 & 0.77 \\
      \hline
 $\mathrm{MTA}$-$\mathrm{RA}$ & 0.1 &  1.1 &  10.2 &              \tblopt{39.0} &  24.4 &  16.8 &   8.4 &   \makecell{9.3 \\ (7.5)} &  \makecell{15.8 \\ (14.8)} &  \makecell{\tblopt{28.8} \\ (21.0)} &  \makecell{22.6 \\ (16.0)} &  \makecell{15.7 \\ (14.1)} &  \makecell{7.8 \\ (12.3)} \\
 \hdashline
               $\mathrm{TS}$ & 0.3 &  1.2 &  11.1 &              \tblopt{36.9} &  24.2 &  16.1 &  10.2 &  \makecell{12.1 \\ (7.2)} &   \makecell{17.4 \\ (9.9)} &  \makecell{\tblopt{24.1} \\ (12.1)} &  \makecell{21.9 \\ (10.8)} &  \makecell{15.0 \\ (10.3)} &   \makecell{9.1 \\ (8.2)} \\
\hdashline
    $\mathrm{TS}\_\mathrm{A}$ & 0.3 &  1.8 &  13.2 &  \tblwinrec{\tblopt{45.6}} &  24.1 &  11.4 &   3.7 &  \makecell{14.2 \\ (9.4)} &  \makecell{22.2 \\ (13.5)} &  \makecell{\tblopt{28.6} \\ (13.7)} &  \makecell{21.0 \\ (12.5)} &  \makecell{10.3 \\ (11.2)} &   \makecell{3.4 \\ (6.9)} \\
\hline
%
\multicolumn{2}{l}{\textbf{Sc. 9:} Tox prob} & 0.25  & 0.43 & 0.50 & 0.58 & 0.64 & 0.75 & 0.25  & \dash{0.43} & \dash{0.50} & \dash{0.58} & \dash{0.64} & \dash{0.75} \\ 
\multicolumn{2}{l}{\textbf{Sc. 9:} Eff prob}  & 0.3  & 0.4 & 0.5 & 0.6 & 0.61 & 0.63 & 0.3  & \underline{0.4} & 0.5 & 0.6 & 0.61 & 0.63 \\ 
\hline
 $\mathrm{MTA}$-$\mathrm{RA}$ & 18.8 &  40.0 &  33.1 &  7.0 &  0.9 &  0.1 &  0.1 &  \makecell{32.0 \\ (30.3)} &  \makecell{\dash{30.3} \\ (24.5)} &  \makecell{\dash{13.6} \\ (15.6)} &  \makecell{\dash{4.3} \\ (7.4)} &  \makecell{\dash{1.0} \\ (3.2)} &  \makecell{\dash{0.1} \\ (0.7)} \\
 \hdashline
               $\mathrm{TS}$ &  \tblwinrec{49.0} &  37.3 &  12.4 &  1.1 &  0.1 &  0.0 &  0.0 &  \makecell{29.0 \\ (31.9)} &         \makecell{13.7 \\ (16.5)} &           \makecell{4.5 \\ (7.5)} &         \makecell{1.9 \\ (3.9)} &         \makecell{1.1 \\ (2.8)} &         \makecell{0.8 \\ (2.1)} \\
\hdashline
    $\mathrm{TS}\_\mathrm{A}$ &  \tblwinrec{50.5} &  39.8 &   9.2 &   0.5 &  0.1 &  0.0 &  0.0 &  \makecell{31.2 \\ (35.1)} &  \makecell{\dash{14.6} \\ (18.6)} &    \makecell{\dash{3.3} \\ (7.0)} &  \makecell{\dash{0.4} \\ (1.9)} &  \makecell{\dash{0.0} \\ (0.3)} &  \makecell{\dash{0.0} \\ (0.2)} \\
\hline
\multicolumn{2}{l}{\textbf{Sc. 10:} Tox prob} & 0.05  & 0.1 & 0.25 & 0.55 & 0.7 & 0.9 & 0.05  & 0.1 & 0.25 & \dash{0.55} & \dash{0.7} & \dash{0.9} \\ 
\multicolumn{2}{l}{\textbf{Sc. 10:} Eff prob} & 0.01  & 0.02 & 0.05 & 0.35 & 0.55 & 0.7 & 0.01  & 0.02 & 0.05 & 0.35 & 0.55 & 0.7 \\ 
\hline 
 $\mathrm{MTA}$-$\mathrm{RA}$ & 91.8 &    0.5 &   0.5 &  2.3 &   4.8 &   0.1 &  0.0 &  \makecell{0.6 \\ (2.8)} &   \makecell{0.7 \\ (2.6)} &    \makecell{1.3 \\ (5.9)} &  \makecell{\dash{4.4} \\ (15.4)} &  \makecell{\dash{1.1} \\ (4.3)} &  \makecell{\dash{0.2} \\ (1.1)} \\
 \hdashline
               $\mathrm{TS}$ & 61.9 &  12.2 &  2.5 &  1.8 &  19.7 &  1.8 &  0.1 &  \makecell{3.8 \\ (6.8)} &  \makecell{8.9 \\ (15.2)} &  \makecell{16.3 \\ (23.0)} &         \makecell{6.3 \\ (10.6)} &         \makecell{1.9 \\ (4.0)} &         \makecell{0.9 \\ (2.3)} \\
\hdashline
    $\mathrm{TS}\_\mathrm{A}$ &  \tblwinrec{94.1} &    0.2 &   0.1 &  1.2 &   4.3 &   0.1 &  0.0 &  \makecell{0.5 \\ (2.2)} &   \makecell{0.6 \\ (2.5)} &    \makecell{1.4 \\ (6.1)} &  \makecell{\dash{2.9} \\ (12.0)} &  \makecell{\dash{0.5} \\ (2.4)} &  \makecell{\dash{0.0} \\ (0.4)} \\
\hline
\end{tabular}
\end{table}       
\begin{table}[p]
\caption{Results for MED identification (part 3/3).}
\label{tbl-eff-b}
\renewcommand{\tabcolsep}{.3em}
\centering
\begin{tabular}{lccccccc|cccccc}
    Algorithm &  E-Stop 
    &\multicolumn{6}{c}{ Recommended} & \multicolumn{6}{c}{Allocated} \\
    \hline
    &
    &   \multicolumn{1}{c}{1}
    &   \multicolumn{1}{c}{2}
    &   \multicolumn{1}{c}{3}
    &   \multicolumn{1}{c}{4}
    &   \multicolumn{1}{c}{5}
    &   \multicolumn{1}{c}{6}
    &   \multicolumn{1}{c}{1}
    &   \multicolumn{1}{c}{2}
    &   \multicolumn{1}{c}{3}
    &   \multicolumn{1}{c}{4}
    &   \multicolumn{1}{c}{5}
    &   \multicolumn{1}{c}{6} \\
\hline
\multicolumn{2}{l}{\textbf{Sc. 11:} Tox prob} & 0.5  & 0.6 & 0.69 & 0.76 & 0.82 & 0.89 & \dash{0.5}  & \dash{0.6} & \dash{0.69} & \dash{0.76} & \dash{0.82} & \dash{0.89} \\ 
\multicolumn{2}{l}{\textbf{Sc. 11:} Eff prob} & 0.4  & 0.55 & 0.65 & 0.65 & 0.65 & 0.65 & 0.4  & 0.55 & 0.65 & 0.65 & 0.65 & 0.65 \\ 
\hline
 $\mathrm{MTA}$-$\mathrm{RA}$ & 90.1 &  9.6 &  0.2 &  0.1 &  0.0 &  0.0 &  0.0 &  \makecell{\dash{7.2} \\ (22.7)} &  \makecell{\dash{2.0} \\ (7.9)} &  \makecell{\dash{0.5} \\ (2.3)} &  \makecell{\dash{0.1} \\ (0.9)} &  \makecell{\dash{0.0} \\ (0.2)} &  \makecell{\dash{0.0} \\ (0.0)} \\
 \hdashline
               $\mathrm{TS}$ &  \tblwinrec{99.8} &   0.2 &  0.0 &  0.0 &  0.0 &  0.0 &  0.0 &          \makecell{0.1 \\ (3.0)} &         \makecell{0.1 \\ (1.2)} &         \makecell{0.0 \\ (0.4)} &         \makecell{0.0 \\ (0.1)} &         \makecell{0.0 \\ (0.0)} &         \makecell{0.0 \\ (0.1)} \\
\hdashline
    $\mathrm{TS}\_\mathrm{A}$ &  \tblwinrec{99.5} &   0.5 &  0.0 &  0.0 &  0.0 &  0.0 &  0.0 &   \makecell{\dash{0.4} \\ (5.6)} &  \makecell{\dash{0.1} \\ (1.6)} &  \makecell{\dash{0.0} \\ (0.2)} &  \makecell{\dash{0.0} \\ (0.0)} &  \makecell{\dash{0.0} \\ (0.0)} &  \makecell{\dash{0.0} \\ (0.0)} \\
\hline
\multicolumn{2}{l}{\textbf{Sc. 12:} Tox prob}  & $0.01$  & $0.02$ & $0.05$ & $\underline{0.1}$ & $0.25$ & $0.5$  & $0.01$  & $0.02$ & $0.05$ & $\underline{0.1}$ & $0.
25$ & \dash{$0.5$} \\
\multicolumn{2}{l}{\textbf{Sc. 12:} Eff prob} & $0.05$  & $0.25$ & $0.45$ & $\underline{0.7}$ & $0.7$ & $0.7$ & $0.05$  & $0.25$ & $0.45$ & $\underline{0.7}$ & $0.7$
& \dash{$0.7$} \\
\hline
 $\mathrm{MTA}$-$\mathrm{RA}$ & 1.0 &  0.1 &  1.2 &   8.9 &              \tblopt{52.8} &  29.4 &  6.4 &  \makecell{5.8 \\ (2.4)} &   \makecell{7.6 \\ (6.5)} &  \makecell{14.6 \\ (15.7)} &  \makecell{\tblopt{35.9} \\ (24.2)} &  \makecell{24.9 \\ (20.0)} &  \makecell{\dash{10.2} \\ (12.8)} \\
 \hdashline
               $\mathrm{TS}$ & 0.8 &  0.0 &   0.7 &  10.0 &  \tblwinrec{\tblopt{57.0}} &  27.7 &  4.0 &  \makecell{7.7 \\ (4.2)} &  \makecell{10.4 \\ (6.5)} &  \makecell{17.9 \\ (10.1)} &  \makecell{\tblopt{32.2} \\ (13.6)} &  \makecell{21.9 \\ (11.9)} &           \makecell{9.3 \\ (7.0)} \\
\hdashline
    $\mathrm{TS}\_\mathrm{A}$ & 1.7 &  0.0 &  1.4 &  10.0 &  \tblwinrec{\tblopt{56.0}} &  26.8 &  4.2 &  \makecell{7.5 \\ (3.9)} &  \makecell{11.3 \\ (8.5)} &  \makecell{19.5 \\ (11.3)} &  \makecell{\tblopt{32.1} \\ (14.4)} &  \makecell{21.6 \\ (13.0)} &    \makecell{\dash{6.4} \\ (7.5)} \\
\hline
\multicolumn{2}{l}{\textbf{Sc. 13:} Tox prob}  & $0.01$  & $0.05$ & $0.1$ & $0.2$ & $\underline{0.3}$ & $0.5$  & $0.01$  & $0.05$ & $0.1$ & $0.2$ & $\underline{0.3}$ & $\dash{0.5}$ \\
\multicolumn{2}{l}{\textbf{Sc. 13:} Eff prob}  & $0.05$  & $0.1$ & $0.2$ & $0.35$ & $\underline{0.55}$ & $0.55$ & $0.05$  & $0.1$ & $0.2$ & $0.35$ & $\underline{0.55}$ & $0.55$ \\
\hline 
 $\mathrm{MTA}$-$\mathrm{RA}$ & 14.9 &  0.7 &  1.8 &  5.5 &  17.0 &  \tblopt{50.3} &  9.7 &  \makecell{6.4 \\ (6.5)} &   \makecell{7.4 \\ (7.3)} &  \makecell{11.1 \\ (12.6)} &  \makecell{18.7 \\ (18.7)} &  \makecell{\tblopt{30.7} \\ (23.8)} &  \makecell{\dash{10.8} \\ (14.0)} \\
 \hdashline
               $\mathrm{TS}$ & 8.6 &  0.5 &  1.8 &  6.7 &  37.6 &  \tblopt{39.0} &  5.6 &  \makecell{9.1 \\ (6.3)} &  \makecell{11.5 \\ (8.2)} &  \makecell{17.5 \\ (12.0)} &  \makecell{26.3 \\ (14.8)} &  \makecell{\tblopt{18.6} \\ (14.2)} &           \makecell{8.4 \\ (7.7)} \\
\hdashline
    $\mathrm{TS}\_\mathrm{A}$ & 17.3 &  0.5 &  1.4 &  7.4 &  31.6 &  \tblopt{37.5} &  4.2 &  \makecell{7.2 \\ (4.5)} &   \makecell{9.1 \\ (6.8)} &  \makecell{16.7 \\ (13.7)} &  \makecell{26.8 \\ (17.1)} &  \makecell{\tblopt{18.1} \\ (15.3)} &    \makecell{\dash{4.7} \\ (7.3)} \\
\hline
\end{tabular}
\end{table}

\section{Revisiting the Treatment versus Experimentation Trade-off}\label{sec:Discussion}

Ideally, a good design for MTD identification should be supported by a control of both the error probability $e_n = \bP(\hat{k}_n \neq k^*)$ and the number of sub-optimal selections $\bE[N_k(n)]$ for $k\neq k^*$. These two quantities are respectively useful to check whether the design achieves a \emph{good identification of the optimal dose} and whether \emph{a large number of patients have been treated with the optimal dose}.   

For classical bandits (in which $k^*$ is the arm with largest mean instead of the MTD), those two performance measures are known to be antagonistic. Indeed, \cite{Bubeckal11} shows that the smaller the regret (a quantity that can be related to the number of sub-optimal selections), the larger the error probability. Such a trade-off may also exist for the MTD identification problem. However, the precise statement of such a result would be meaningful for large values of the number of patients $n$, which is of little interest for a real clinical trial as it can only involve a small number of patients. In practice, we showed that adaptations of Thompson Sampling, a bandit design aimed at maximizing rewards, achieve good performance in terms of both allocation and recommendation. 

Still, another natural avenue of research is to investigate the adaptation of bandit designs aimed at minimizing the error probability. Minimizing the error probability for MTD can be viewed as a variant of the fixed-budget Best Arm Identification (BAI) problem introduced by \cite{Bubeck10BestArm,Bubeckal11}. In contrast to the standard BAI problem that aims to identify the arm with largest mean (which would correspond here to the most toxic dose), the focus is on identifying the arm whose mean is closest to the threshold $\theta$. A state-of-the art fixed-budget BAI algorithm is Sequential Halving \citep{Karnin13}, and we propose in Algorithm~\ref{alg-ST} a natural adaptation to MTD identification. 

Sequential Halving for MTD identification proceeds in phases. In each of the $\log_2(K)$ phases, all the remaining doses are allocated the same amount of times to patients and their empirical toxicity based on these allocations (that is, the average of the toxicity responses) is computed. At the end of each phase the empirical worst half of the doses is eliminated. For MTD identification, rather than the doses with the smallest empirical means (as the vanilla Sequential Halving algorithm would do), the doses whose empirical toxicity are the furthest away from the threshold $\theta$ are eliminated. Observe that by design of the algorithm, the total number of allocated doses is indeed smaller than the prescribed budget $n$.

\begin{algorithm}[h!]
\label{alg-ST}
\begin{algorithmic}
\State \textbf{Input:} budget $n$, target toxicity $\theta$
\State \textbf{Initialization:} Set of dose levels $S_0 \leftarrow \{1, \dots, K\}$;
  \For{$r \leftarrow 0$ to $\lceil\log_2(K)\rceil-1$}
    \State Allocate each dose $k \in S_r$ to $t_r = \left\lfloor\frac{n}{|S_r|\lceil\log_2(K)\rceil}\right\rfloor$ patients; 
    \State Based on their response compute $\hat p_k^r$, the empirical toxicity of dose $k$ based on these $t_r$ samples;
    \State Compute $S_{r+1}$ the set of $\lceil|S_r|/2\rceil$ arms with 
        smallest $\hat d_k^r:=|\theta - \hat{p}_k^r|$
  \EndFor
\State \textbf{Output:} the unique arm in  $S_{\lceil{\log_2(K)}\rceil}$
\end{algorithmic}
\caption{Sequential Halving for MTD Identification}
\end{algorithm}


Building on the analysis of \cite{Karnin13}, one can establish the following upper bound on the error probability of Sequential Halving for MTD identification. The proof can be found in Appendix~\ref{sec:SHproof}. 

\begin{theorem}\label{thm:SH}
The error probability of the \texttt{SH} algorithm is upper bounded as 
\begin{align*}
  \bP\left(\hat{k}_n \neq k^*\right) \leq 9 \log_2 K \cdot \exp \left(
  - \frac{n}{8 H_2 (\bm p)\log_2 K}
  \right),
\end{align*}
where $H_2(\bm p):= \max_{k \ne k^*} {k}{\Delta_{[k]}^{-2}}$ where $\Delta_k = |p_k - \theta| - |p_{k^*} - \theta|$ and $\Delta_{[1]} \leq \Delta_{[2]} \leq \dots \leq \Delta_{[K]}$.
\end{theorem}
A consequence of Theorem~\ref{thm:SH} is that in a trial involving more than $n = 8 H_2(\bm p) \log_2 K\log\left(9\log_2(K)/\delta\right)$ patients, Sequential Halving is guaranteed to identify the MTD with probability larger than $1-\delta$. However, this number is typically much larger than the number of patients involved in a clinical trial. Indeed the complexity term $H_2(\bm p)$ may be quite large, when some doses have a distance to the threshold $\theta$ which is very close to the smallest distance $|p_{k^*} - \theta|$. 

An important shortcoming of Sequential Halving is that due to the uniform exploration within each phase each dose is selected at least $n / (K\log_2(K))$ times, even the largest, possibly harmful ones. This is highly unethical in a  clinical trial without prior knowledge that too-toxic (or too ineffective) doses have already been eliminated. This problem of allocating too extreme doses is likely to be shared by adaptations of any other BAI algorithm, that are expected to select all the arms a linear number of times. For example the APT algorithm proposed by \cite{Locatelli16Thres} to identify all arms with mean above a threshold $\theta$ using a fixed budget $n$ also selects all arms a linear number of times.

To overcome this problem, an interesting avenue of research would be to try to incorporate monotonicity assumptions in BAI algorithms. \cite{Garivier17DF} recently proposed such an algorithm, in the fixed confidence setting: given a risk parameter $\delta$, the goal is to identify a dose $\hat{k}_\tau$  such that $\bP(\hat{k}_\tau \neq k^*) \leq \delta$, using as few samples $\tau$ as possible. Their analysis identifies a minimal sample complexity $\bE[\tau]$ that guarantees a $\delta$-correct identification for any increasing toxicities, which can be obtained under an \emph{optimal allocation} $w^*$ (where $w_k^*$ indicates the fraction of time dose $k$ is allocated). Interestingly, this optimal allocation is supported only on the neighboring doses of the MTD. The fixed-confidence setting requires allowing for random stopping rules $\tau$, i.e. for a dose-finding trial based on an adaptively chosen number of patients. This is not always possible in practice, and it would be interesting to investigate optimal allocations in a fixed-budget setting as well. Yet optimality in the fixed-budget setting is a notoriously hard question  already for classical bandits \citep{Locatelli16LBFB}.


\section{Conclusion}\label{sec:Conclusion}

Motivated by the literature on multi-armed bandit models, we advocated the use of the powerful Thompson Sampling principle for dose-finding studies. This Bayesian randomized algorithm can be used in different contexts as it can leverage different prior information about the doses. For increasing toxicities and increasing or plateau efficacies, we proposed variants of Thompson Sampling, notably the $\mathrm{TS}\_\mathrm{A}$ algorithm that often outperforms our baselines in terms of recommendation of the optimal dose, while significantly reducing the allocation to doses with high toxicity. 

We provided theoretical guarantees for the simplest version of Thompson Sampling based on independent uniform priors on each dose toxicity, but advocated the use of more sophisticated priors for practical dose-finding studies. We believe that finding a practical design for which we can also establish non-trivial finite-time performance guarantees is a crucial research question.

Another interesting direction would be taking contextual information (e.g. a patient's medical history and other medications used) into account for a more ``personalized'' assessment of toxicity and efficacy of a drug. Bayesian methods also seem promising for such an objective, following the success of Thompson Sampling for contextual bandits.

%
%




\section*{Acknowledgments} Emilie Kaufmann acknowledges the support of the French Agence Nationale de la Recherche (ANR) under grant ANR-16-CE40-0002 (BADASS project) and ANR-19-CE23-0026-04 (BOLD project). We thank anonymous reviewers of this paper for their helpful suggestions for improvements. 






\bibliographystyle{biom}
\bibliography{biblioBandits}

\appendix

\section{Analysis of Independent Thompson Sampling: Proof of Theorem~\ref{thm:TS}} \label{proof:TS}

Fix a sub-optimal arm $k$. Several cases need to be considered depending on the relative position of $p_k$ and $p_{k^*}$ with respect to the threshold. All cases can be treated similarly and to fix the ideas, we consider the case $p_{k^*} \geq \theta > p_k$, which is illustrated below. In that case $d^*_k = 2\theta - p_{k^*}$ satisfies $p_k < d_{k}^* \leq \theta$.

\begin{figure}[h]\centering
 \includegraphics[height=7cm,angle=-90]{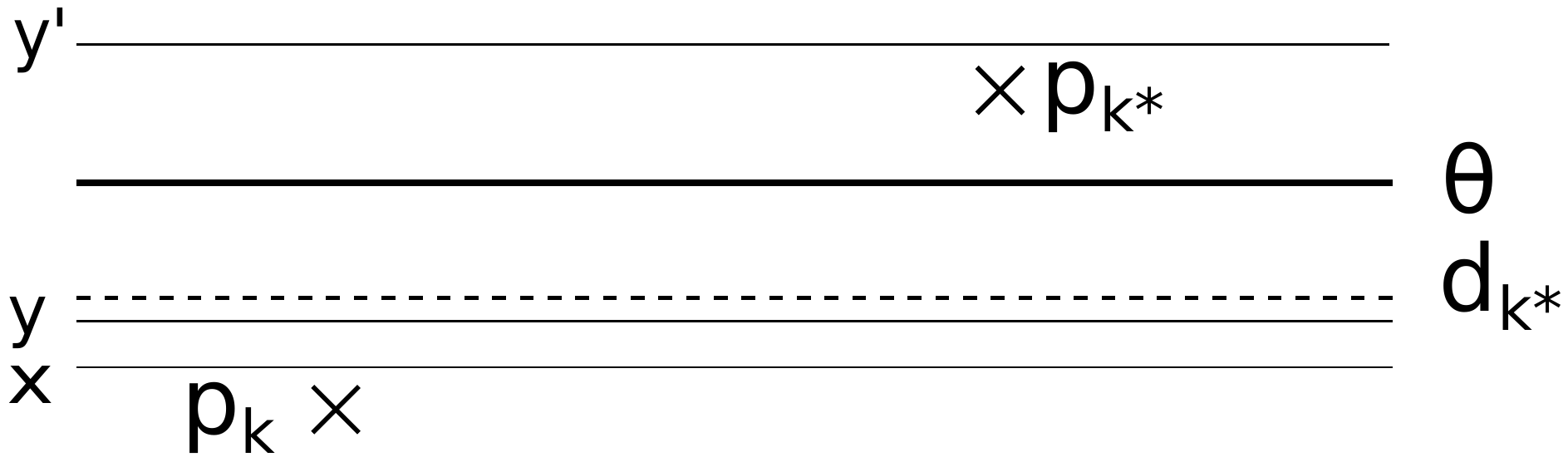}
\end{figure}

Let $x,y \in ]0,1[^2$ be such that $p_k < x < y < d_{k^*}$, that will be chosen later. Define $y' = 2\theta - y > \theta$ the symmetric of $y$ with respect to the threshold (see the above illustration). We denote by  $\hat{\mu}_k(t)$ the empirical mean of the toxicity responses gathered from dose $k$ up to the end of round $t$ and recall $\theta_k(t)$ is the sample from the Beta posterior on $p_k$ after $t$ rounds that is used in the Thompson Sampling algorithm. Inspired by the analysis of \cite{AGAISTAT13}, we introduce the following two events, that are quite likely to happen when enough samples of arm $k$ have been gathered: 
\begin{eqnarray*}
 E_k^\mu (t) & = & \left(\hat{\mu}_k(t) \leq x\right) \ \ \ \text{and} \ \ \ E_k^\theta (t)  =  \left({\theta}_k(t) \leq y\right).
\end{eqnarray*}

The expected number of allocations of dose $k$ is then decomposed in the following way 
\begin{align*}
 \bE[N_k(T)] = &\underbrace{\sum_{t=0}^{T-1}\bP\left(D_{t+1} = k, E_k^\mu(t),E_k^\theta(t)\right)}_{(I)}+ \underbrace{\sum_{t=0}^{T-1}\bP\left(D_{t+1} = k, E_k^\mu(t),\overline{E_k^\theta(t)}\right)}_{(II)}
 \\ &
 	+ \underbrace{\sum_{t=0}^{T-1}\bP\left(D_{t+1} = k, \overline{E_k^\mu(t)}\right)}_{(III)}  
\end{align*}
Terms (II) and (III) are easily controlled using some concentration inequalities and the so-called Beta-Binomial trick, that is the fact that the CDF of a Beta distribution with parameters $a$ and $b$, $F^{\text{Beta}}_{a,b}$, is related to the CDF of a binomial distribution with parameter $n,x$, $F^{B}_{n,x}$, in the following way: \[F^{\text{Beta}}_{a,b}(x) = 1 - F^{B}_{a+b-1, x}(a - 1).\] Term (III) is very small as arm $k$ is unlikely to be drawn often while its empirical mean falls above $x > p_k$ and term (II) grows logarithmically with $T$. More precisely, it can be shown using Lemma 3 and 4 in \cite{AGAISTAT13} that 
\[
 (II)  \leq  \frac{\log(T)}{\kl(x,y)} + 1 \ \ \ \text{ and } \ \ \ (III)  \leq  \frac{1}{\kl(x,y)} + 1.
\] 
The tricky part of the analysis is to control term (I), that is to upper bound the number of selections of dose $k$ when both the empirical mean and the Thompson sample for dose $k$ fall close to the true mean $p_k$. For this purpose, one can prove a counterpart of Lemma 1 in \cite{AGAISTAT13} that relates the probability of selecting dose $k$ to that of selecting the MTD $k^*$. 

\begin{lemma}\label{lem:CrucialAG}Define $p_{y}(t) : = \bP\left({\theta}_{k^*}(t) \in [y,y'] | \cF_{t}\right)$, where $\cF_{s}$ is the filtration generated by the observation up to the end of round $s$. Then 
\begin{align*}
&\bP\left(D_{t+1} = k | E_{k}^\theta(t+1),\cF_{t}\right)\leq \frac{1-p_{y}(t)}{p_y(t)}\bP\left(D_{t+1} = k^* | E_k^\theta(t+1),\cF_{t}\right).
\end{align*}
\end{lemma}

\begin{proof} The proof is inspired of that of Lemma 1 in \cite{AGAISTAT13}. We introduce the event in which the Thompson sample for dose $k$ is the closest to the threshold $\theta$ among all sub-optimal doses:
\[M_k(t) = \{ |\theta - \theta_k(t)| \geq |\theta - \theta_\ell(t)| \forall \ell \neq k^*\}.\]
On the one hand, one has
\begin{align*}
 \bP\left(D_{t+1} = k^* | E_k^\theta(t+1),\cF_t\right)& \geq
 	\bP\left(D_{t+1} = k^*,M_k(t) | E_k^\theta(t+1),\cF_t\right)
 \\ & \geq
   \bP\left(\theta_{k^*}(t) \in [y,y'],M_k(t) | E_k^\theta(t+1),\cF_t\right) 
\\ & =
	p_y(t) \times \bP\left(M_k(t) | E_k^\theta(t+1),\cF_t\right).
\end{align*}
On the other hand, it holds that 
\begin{align*} \bP\left(D_{t+1} = k | E_k^\theta(t+1),\cF_t\right)
 & \leq \bP\left(\theta_{k^*}(t) \notin [y, y'], M_k(t) | E_k^\theta(t+1),\cF_t\right)
\\ & = (1 - p_y(t)) \times \bP\left(M_k(t) | E_k^\theta(t+1),\cF_t\right).
\end{align*}
Combining the two inequalities yields Lemma~\ref{lem:CrucialAG}.
\end{proof}

Using the same steps as \cite{AGAISTAT13} yields an upper bound on the first term:
\[(I) \leq \sum_{j=1}^{T-1}\bE\left[\frac{1}{p_y(\tau_{j})} - 1\right],\]
where $\tau_j$ is the time instant at which dose $k$ is selected for the $j$-th time. The expectation of $1/p_y(\tau_{j})$ can be explicitly written 
\[\bE\left[\frac{1}{p_y(\tau_{j})}\right] =\sum_{s=0}^j \frac{f^B_{j,p_{k^*}}(s)}{\bP\left(y \leq X_{s+1,j-s+1} \leq y'\right)} \]
where $f^B_{n,x}$ stands for the pdf of a Binomial distribution and $X_{a,b}$ denotes a random variable that has a $\mathrm{Beta}(a,b)$ distribution. The following lemma is crucial to finish the proof. This original result was specifically obtained for the MTD identification problem and is needed to control the probability that a Beta distributed random variable fall inside an interval, that is $\bP\left(y \leq X_{s+1,j-s+1} \leq y'\right)$.

\begin{lemma}\label{lem:Technical} There exists $j_0$ such that, for all $j \geq j_0$,
\begin{align*}
&\forall s \in \{0,\dots, j\}, \ \ \bP\left(y \leq X_{s+1,j-s+1} \leq y'\right) \geq \frac{1}{2}\min  \left\{\bP\left(X_{s+1,j+s+1} \geq y\right) ; \bP\left(X_{s+1,j+s+1} \leq y'\right)\right\}
\end{align*}
\end{lemma}

Using Lemma~\ref{lem:Technical} and the Beta-Binomial trick, one can write, for $j \geq j_0$, 
\begin{align}
\nonumber
\bE\left[\frac{1}{p_y(\tau_{j})}\right]
 & \leq
\sum_{s=0}^j \frac{2f^B_{j,p_{k^*}}(s)}{\bP\left(X_{s+1,j+s+1} \geq y\right)}+\sum_{s=0}^j \frac{2f^B_{j,p_{k^*}}(s)}{\bP\left(X_{s+1,j+s+1} \leq y'\right)} 
 \\ \nonumber & =
 \sum_{s=0}^j \frac{2f^B_{j,p_{k^*}}(s)}{F^B_{j+1,y}(s)} + \sum_{s=0}^j \frac{2f^B_{j,p_{k^*}}(s)}{1-F^B_{j+1,y'}(s)}
 \\ & = \sum_{s=0}^j \frac{2f^B_{j,p_{k^*}}(s)}{F^B_{j+1,y}(s)} + \sum_{s=0}^j \frac{2f^B_{j,1-p_{k^*}}(s)}{F^B_{j+1,1-y'}(s)},\label{eq:star}
\end{align}
where the last equality relies on the following properties of the Binomial distribution \[f^B_{n,x}(s) = f^B_{n,1-x}(n-s) \ \ \text{and} \ \ F^B_{n,x}(s) = 1 - F_{n,1-x}(n-s-1)\]
and a change of variable in the second sum. 

Now the following upper bound can be extracted from the proof of Lemma 3 in \cite{AGAISTAT13}. 

\begin{lemma}\label{lem:UB} Fix $u$ and $v$ such that $u < v$ and let $\Delta = v - u$. Then 
\begin{align*}
\sum_{s=0}^j \frac{f^B_{j,v}(s)}{F^B_{j,u}(s)} \leq
 \left\{\begin{array}{cl}
 1 + \frac{3}{\Delta} & \text{if } j < {8}/\Delta,
 \\
 1 + \Theta\left(e^{-\Delta^2 j /2} + \frac{1}{(j+1)\Delta^2}e^{-2\Delta^2 j} + \frac{1}{e^{\Delta^2j/4} - 1}\right) & \text{else.} 
\end{array}
\right.
\end{align*}
\end{lemma}

Each of the two sums in \eqref{eq:star} can be upper bounded using Lemma~\ref{lem:UB}. Letting $\Delta_1 = p_{k^*} - y$ and $\Delta_2 = y' - p_{k^*}$, one obtains  
\begin{align*}
 (I) \leq &
 	\sum_{j=1}^{j_0} \bE\left[\frac{1}{p_y(\tau_{j})} \right]
 	- j_0 + \frac{24}{\Delta_1^2} + \frac{24}{\Delta_2^2} 
\\ &
	+ C \sum_{j=0}^{T-1} \left[e^{-\Delta_1^2 j /2} + \frac{1}{(j+1)\Delta_1^2}e^{-2\Delta_1^2 j} + \frac{1}{e^{\Delta_1^2j/4} - 1}\right]
\\ &
	+ C \sum_{j=0}^{T-1} \left[e^{-\Delta_2^2 j /2} + \frac{1}{(j+1)\Delta_2^2}e^{-2\Delta_2^2 j} + \frac{1}{e^{\Delta_2^2j/4} - 1}\right], 
\end{align*}
which is a constant (as the series have a finite sum) that only depends on $y, \theta$ and $p_{k^*}$ (through $y'$ and the gaps $\Delta_1$ and $\Delta_2$ defined above). 

Putting things together, we proved that for every $x$ and $y$ satisfying $p_k < x < y < d_{k^*}$, the number of selections of dose $k$ is upper bounded as 
\[ \bE[N_k(T)] \leq \frac{1}{\kl(x,y)}\log(T) + C_{x,y,\theta,\bm p}\]
for some constant that depends on the toxicity probabilities, the threshold $\theta$ and the choice of $x$ and $y$. Now, picking $x$ and $y$ such that $\kl(x,y) = \frac{\kl(p_k,d_{k^*})}{1+\epsilon}$ yield the result. 

\qed 

\paragraph{Proof of Lemma~\ref{lem:Technical}.} The proof uses the two equalities below
\begin{align}
 \bP\left(y \leq X_{s+1,j-s+1} \leq y'\right) & = \bP\left(X_{s+1,j-s+1} \geq y\right) - \bP\left(X_{s+1,j-s+1} \geq y'\right)\label{ForSmallS}
\\ 
\bP\left(y \leq X_{s+1,j-s+1} \leq y'\right)
 & = \bP\left(X_{s+1,j-s+1} \leq y'\right) - \bP\left(X_{s+1,j-s+1} \leq y\right),\label{ForLargeS}
\end{align}
as well as the Sanov inequalities: if $S_{n,x}$ is a binomial distribution with parameters $n$ and $x$, then 
\begin{align}
\nonumber
\frac{e^{-n\kl(k/n,x)}}{n+1}
& \leq \bP\left(S_{n,x} \geq k\right) 
\\ & \leq e^{-n\kl(k/n,x)} \ \ \text{if } \ k > xn \label{Sanov}
\\ \nonumber
\frac{e^{-n\kl(k/n,x)}}{n+1}
& \leq \bP\left(S_{n,x} \leq k\right)
\\ & \leq
	e^{-n\kl(k/n,x)} \ \ \text{if } \ k < xn \label{SanovMin}
\end{align}

We prove the inequality considering 4 cases. We define $\ymid = \frac{y+y'}{2}$. 

\paragraph{Case 1: $\bm{s < (j+1)y}$} Starting from equality \eqref{ForSmallS} and using the Beta-Binomial trick yields  
\begin{align*}
&\bP\left(y \leq X_{s+1,j-s+1} \leq y'\right) = \bP\left(S_{j+1,y} \leq s\right) - \bP\left(S_{j+1,y'} \leq s\right).
\end{align*}
Using Sanov inequalities, we shall prove that there exists some $j_1$ such that if $j\geq j_1$, 
\[\forall s \leq (j+1)y, \ \ \bP\left(S_{j+1,y'} \leq s\right) \leq \frac{1}{2}\bP\left(S_{j+1,y} \leq s\right).\]
As $s$ is smaller than the mean of the two Binomial distributions, by \eqref{SanovMin} it is sufficient to prove that 
\begin{align*}
&\forall s \leq (j+1)y,
\ \ e^{-(j+1)\kl\left(\frac{s}{j+1} , y'\right)} \leq \frac{1}{2(j+2)}e^{-(j+1)\kl\left(\frac{s}{j+1} , y\right)}
\end{align*}
which in turn is equivalent to 
\begin{align*}
&\forall s \leq (j+1)y,
\ \ \kl\left(\frac{s}{j+1} , y'\right)  -  \kl\left(\frac{s}{j+1} , y\right) \geq \frac{\log(2(j+2))}{j+1}.
\end{align*}
As the function in the left-hand side is non-increasing in $s$, a sufficient condition is that $j$ satisfies 
\[ \kl\left(y , y'\right)\geq \frac{\log(2(j+2))}{j+1},\]
which is the case for $j$ superior to some $j_1$. Thus, for $j\geq j_1$, 
\begin{align*}
\bP\left(y \leq X_{s+1,j-s+1} \leq y'\right)\geq \frac{1}{2}\bP\left(S_{j+1,y} \leq s\right)= \frac{1}{2}\bP\left(X_{s+1,j-s+1} \geq y\right).
\end{align*}

\paragraph{Case 2: $\bm{(j+1)y \leq s \leq (j+1)\ymid}$} Starting from equality \eqref{ForSmallS} and using the Beta-Binomial trick and the upper bound in \eqref{SanovMin} yields 
\begin{align*}
\bP\left(y \leq X_{s+1,j-s+1} \leq y'\right) 
 & \geq \bP\left(S_{j+1,y} \leq s\right) - e^{-(j+1)\kl\left(\frac{s}{j+1} , y'\right)}
\\ & \geq \bP\left(S_{j+1,y} \leq s\right) - e^{-(j+1)\kl\left(\ymid , y'\right)}. 
\end{align*}
The median of $S_{j+1,y}$ is $\lfloor(j+1)y\rfloor$ or $\lceil(j+1)y\rceil$. As $s \leq  (j+1)y$, it holds that $\bP\left(S_{j+1,y} \leq s\right) \geq \frac{1}{2}$. Therefore, for all $j \geq j_2 := \frac{\ln 4}{\kl(\ymid,y')}-1$, 
\[e^{-(j+1)\kl\left(\ymid , y'\right)} \leq \frac{1}{4} \leq \frac{1}{2}\bP\left(S_{j+1,y} \leq s\right).\]
Therefore if $j \geq j_2$, $\bP\left(y \leq X_{s+1,j-s+1} \leq y'\right)  \geq \frac{1}{2}\bP\left(X_{s+1,j-s+1} \geq y\right)$.

\paragraph{Case 3: $\bm{(j+1)\ymid \leq s \leq (j+1)y'}$} Starting from equality \eqref{ForLargeS} and using the Beta-Binomial trick and the upper bound in \eqref{Sanov} yields 
\begin{align*}
 \bP\left(y \leq X_{s+1,j-s+1} \leq y'\right) & \geq \bP\left(S_{j+1,y'} \geq s\right) - e^{-(j+1)\kl\left(\frac{s}{j+1} , y\right)}
\\ & \geq \bP\left(S_{j+1,y'} \geq s\right) - e^{-(j+1)\kl\left(\ymid , y\right)}. 
\end{align*}
The median of $S_{j+1,y'}$ is $\lfloor(j+1)y'\rfloor$ or $\lceil(j+1)y'\rceil$. As $s \leq  (j+1)y'$, it holds that $\bP\left(S_{j+1,y'} \geq s\right) \geq \frac{1}{2}$. Therefore, for all $j \geq j_3 := \frac{\ln 4}{\kl(\ymid,y)}-1$, 
\[e^{-(j+1)\kl\left(\ymid , y\right)} \leq \frac{1}{4} \leq \frac{1}{2}\bP\left(S_{j+1,y'} \geq s\right).\]
Therefore if $j \geq j_3$, $\bP\left(y \leq X_{s+1,j-s+1} \leq y'\right)  \geq \frac{1}{2}\bP\left(X_{s+1,j-s+1} \leq y'\right)$.

\paragraph{Case 4: $\bm{s > (j+1)y'}$} Starting from equality \eqref{ForLargeS} and using the Beta-Binomial trick yields  
\[\bP\left(y \leq X_{s+1,j-s+1} \leq y'\right)  =  \bP\left(S_{j+1,y'} \geq s\right) - \bP\left(S_{j+1,y} \geq s\right).\]
Using Sanov inequalities, we shall prove that there exists some $j_4$ such that if $j\geq j_4$, 
\[\forall s \geq (j+1)y', \ \ \bP\left(S_{j+1,y} \geq s\right) \leq \frac{1}{2}\bP\left(S_{j+1,y'} \geq s\right).\]
As $s$ is larger than the mean of the two Binomial distributions, by \eqref{Sanov} it is sufficient to prove that 
\[\forall s \geq (j+1)y', \ \ e^{-(j+1)\kl\left(\frac{s}{j+1} , y\right)} \leq \frac{1}{2(j+2)}e^{-(j+1)\kl\left(\frac{s}{j+1} , y'\right)}\]
which in turn is equivalent to 
\[\forall s \geq (j+1)y', \ \ \kl\left(\frac{s}{j+1} , y\right)  -  \kl\left(\frac{s}{j+1} , y'\right) \geq \frac{\log(2(j+2))}{j+1}.\]
As the function in the left-hand side is non-decreasing in $s$, a sufficient condition is that $j$ satisfies 
\[ \kl\left(y' , y\right)\geq \frac{\log(2(j+2))}{j+1},\]
which is the case for $j$ superior to some $j_4$. Thus, for $j\geq j_4$, 
\begin{align*}
\bP\left(y \leq X_{s+1,j-s+1} \leq y'\right) 
&\geq \frac{1}{2}\bP\left(S_{j+1,y'} \geq s\right) 
\\ &= \frac{1}{2}\bP\left(X_{s+1,j-s+1} \leq y'\right).
\end{align*}

\paragraph{Conclusion} Letting $j_0 = \max(j_1,j_2,j_3,j_4)$, for all $j\geq j_0$, for every $s \in \{0,\dots,j\}$,
\begin{align*}
&\bP\left(y \leq X_{s+1,j-s+1} \leq y'\right) \geq \frac{1}{2}\min \left\{\bP\left(X_{s+1,j+s+1} \geq y\right) ; \bP\left(X_{s+1,j+s+1} \leq y'\right)\right\}
\end{align*}

\section{Lower Bound on the Number of Allocation: Proof of Theorem~\ref{thm:LB}}\label{proof:LB}

Fix a uniformly efficient algorithm and a vector of toxicity probabilities $\bm p$. We denote by $\bE_{\bm p}$ the expectation under the model parameterized by $\bm p$ when this algorithm is used. Letting $\bm p'$ be another vector of probabilities, it follows from the change-of-distribution lemma of \cite{GMS18} that for all random variable $Z_T \in [0,1]$ which is $\cF_T$-measurable 
\begin{equation}\sum_{\ell=1}^K \bE_{\bm p}[N_{\ell}(T)]\kl\left(p_{\ell}, p'_{\ell}\right) \geq \kl\left(\bE_{\bm p}[Z_T], \bE_{\bm p'} [Z_T]\right).\label{eq:CD}\end{equation}
Letting $k^*$ be a MTD in $\bm p$, we fix $k$ which is not a MTD (i.e. $|p_k - \theta | > |p_{k^*} - \theta|$) and we prove that 
\begin{equation}\liminf_{T\rightarrow \infty}\frac{\bE_{\bm p}[N_{k}(T)]}{\ln(T)} \geq \frac{1}{\kl(p_k,d_k^*)}\;.\label{ToProveLB}\end{equation} 

Recall that we assume $p_{k^*} \neq \theta$. Then one can define the alternative model $\bm p'$ in which for all $\ell \neq k$, $p'_{\ell} = p_{\ell}$ and $p'_{k} = d_{k}^* + \epsilon$ if $d_k^* < \theta$ and $p'_{k} = d_{k}^* - \epsilon$ if $d_k^* > \theta$, with $\epsilon$ small enough such that under $\bm p'$, dose $k$ is the unique MTD (refer to Figure~\ref{fig:doses} for an illustration). 

For this particular choice of alternative model $\bm p'$, \eqref{eq:CD} becomes
\begin{eqnarray*}\bE_{\bm p} [N_k(T)] \kl(p_k,d^*_k \pm \epsilon) &\geq& \kl\left(\bE_{\bm p}[Z_T], \bE_{\bm p'} [Z_T]\right)\\
 & \geq & \left(1 -\bE_{\bm p}[Z_T]\right) \ln \left(\frac{1}{1-\bE_{\bm p'} [Z_T]}\right) - \ln(2)
\end{eqnarray*}
Choosing $Z_T = \frac{N_{k}(T)}{T}$, exploiting the fact that the algorithm is uniformly efficient we know that 
\begin{itemize}
 \item $\lim_{T\rightarrow \infty}\bE_{\bm p}[Z_T] = 0$ as $k$ is a sub-optimal dose under $\bm p$ 
 \item $\frac{1}{1-\bE_{\bm p'} [Z_T]} = \frac{T}{T - \bE_{\bm p'}[N_k(T)]} = \frac{T}{\sum_{\ell \neq k} \bE_{\bm p'}[N_{\ell}(T)]}$ and $\sum_{\ell \neq k} \bE_{\bm p'}[N_{\ell}(T)] = o(T^\alpha)$ for all $\alpha \in (0,1)$ as $k$ is the only MTD under $\bm p'$, which yields, for all $\alpha \in (0,1)$, 
 \[\lim_{T\rightarrow \infty} \frac{1}{\ln(T)} \ln \left(\frac{1}{1-\bE_{\bm p'} [Z_T]}\right) \geq 1 - \alpha\;.\]
\end{itemize}
Letting $\alpha$ go to zero, we obtain  
\[\liminf_{T\rightarrow \infty} \frac{\bE_{\bm p} [N_k(T)] \kl(p_k,d^*_k \pm \epsilon)}{\ln(T)} \geq 1\]
and \eqref{ToProveLB} follows by letting $\epsilon$ go to zero. 


\section{Analysis of Sequential Halving: Proof of Theorem~\ref{thm:SH}}
\label{sec:SHproof}

Recall $\hat{d}_k^r = |\theta - \hat{p}_k^t|$ is the empirical distance from the toxicity of dose $k$ to the threshold, where $\hat{p}_k^r$ is the empirical average of the toxicity responses observed for dose $k$ during phase $r$ (based on $t_r$ samples). The central element of the proof is Lemma~\ref{lemma-inversion} below, that controls the probability that dose $k$ seems to be be closer to the threshold than the MTD $k^*$ in phase $r$. Its proof is more sophisticated than that of Lemma 4.2 in \cite{Karnin13} as several cases need to be considered. 
%
%
\begin{lemma}
\label{lemma-inversion}
Assume that the arm closest to $\theta$ was not eliminated
prior to round $r$.
Then for any arm $k \in S_r$,
\begin{equation}\bP(\hat{d}^r_{k^*} > \hat{d}^r_k) \leq 3 \exp\left(- \frac{t_r}{2}\Delta_k^2\right).\label{toproofin4case}\end{equation}
\end{lemma}
%
%
\begin{proof}
For the means $p_{k^*}$ and $p_k$ let $\hat{p}^r_{k^*}$ and $\hat{p}^r_k$ denote their
expected rewards in round $r$, respectively.
We will first derive a probability bound which does not depend on the
ordering of $p_k$ and $p_{k^*}$ w.r.t. $\theta$, and then we will do a case
analysis of the possible orderings to produce our final bound.

The error event can be decomposed as follows. 
\begin{align*}
&\Set{\hat{d}^r_{k^*} > \hat{d}^r_k} =
\\
	&~~~\left( \Set{\hat{p}_{{k^*},r} > \theta}
		\cap \Set{\hat{p}_{k,r} > \theta}
		\cap \Set{\hat{p}_{{k^*},r} - \theta > \hat{p}_{k,r} - \theta}
		\right) \\
	&\cup \left(
		\Set{\hat{p}_{{k^*},r} \le \theta}
		\cap \Set{\hat{p}_{k,r} > \theta}
		\cap \Set{\theta - \hat{p}_{{k^*},r} > \hat{p}_{k,r} - \theta}
		\right) \\
	&\cup \left(
		\Set{\hat{p}_{{k^*},r} > \theta}
		\cap \Set{\hat{p}_{k,r} \le \theta}
		\cap \Set{\hat{p}_{{k^*},r} - \theta > \theta - \hat{p}_{k,r}}
		\right) \\
	&\cup \left(
		\Set{\hat{p}_{{k^*},r} \le \theta}
		\cap \Set{\hat{p}_{k,r} \le \theta}
		\cap \Set{\theta - \hat{p}_{{k^*},r} > \theta - \hat{p}_{k,r}}
		\right)
\end{align*}
From there, we distinguish two cases, in which we show the error event is included in a reunion of events whose probability can be controlled using the Hoeffding's inequality. 

\paragraph{Case 1: $\bm{p_k \geq \theta}$.} In that case, it is very unlikely that $\{\hat{p}_{k,r} < \theta\}$. Hence, we can isolate that event and use the previous decomposition to write
\begin{align*}
&\Set{\hat{d}^r_{k^*} > \hat{d}^r_k} \subseteq 
\\
&\Set{\hat{p}_{k,r} \le \theta} \cup \Set{\hat{p}_{{k^*},r} - \hat{p}_{k,r} > 0}
		\cup \Set{\hat{p}_{k,r} + \hat{p}_{{k^*},r} < 2\theta}
		.
\end{align*}
When $p_k \geq \theta$, irrespective of the position of $p_{k^*}$ with respect to $\theta$, one can justify that $p_k > \theta$, $p_{k^*} - p_{k} < 0$ and ${p}_{k} + {p}_{{k^*}} > 2\theta$ (as $p_k \geq \max(p_{k^*},2\theta - p_{k^*})$ because $k$ is a suboptimal arm larger than the threshold). Therefore, the above three events are unlikely. More precisely, using Hoeffding's inequality yields  
\begin{align*}
\bP(\hat{d}^r_{k^*} > \hat{d}^r_k) 
 &\leq 
	\bP(\hat{p}_{k,r} \le \theta)
	+ \bP(\hat{p}_{{k^*},r} - \hat{p}_{k,r} > 0)
	+ \bP(\hat{p}_{k^*,r} + p_{k,r} < 2\theta)
\\ &\leq 
	\exp\left( -2t_r (\theta - p_k)^2 \right\}
	+ \exp\left\{ - \frac{t_r}{2} (p_{k^*} - p_k)^2 \right\}
\\
	& \hspace{0.4cm} + \exp\left\{ - \frac{t_r}{2} (p_{k^*} + p_k - 2\theta )^2 \right)
\\ &\leq 
	3 \exp\left( -\frac{t_r}{2} \min\left\{(p_{k} - \theta)^2,
		(p_{k} - p_{k^*})^2,
		(p_{k^*}+ p_k - 2\theta)^2
	  \right\} \right)
\\ & = 
	3 \exp\left( -\frac{t_r}{2} \min\left\{
		(p_{k} - p_{k^*})^2,
		(p_{k} - (2\theta - p_{k^*}))^2
	  \right\} \right)
\end{align*}
Equation~\eqref{toproofin4case} follows as $\Delta_k^2 = \min\left\{
		(p_{k} - p_{k^*})^2,
		(p_{k} - (2\theta - p_{k^*}))^2
	  \right\}$. 

\paragraph{Case 2: $\bm{p_k \leq \theta}$.} In that case, the unlikely event is $\{\hat{p}_{k,r} > \theta\}$ and we write 
\begin{align*}
&\Set{\hat{d}^r_{k^*} > \hat{d}^r_k} \subseteq 
\Set{\hat{p}_{k,r} > \theta} \cup \Set{\hat{p}_{{k},r} - \hat{p}_{k^*,r} > 0}\cup \Set{\hat{p}_{k,r} + \hat{p}_{{k^*},r} > 2\theta}.
\end{align*}
When $p_k < \theta$, irrespective of the position of $p_{k^*}$ with respect to $\theta$, one can justify that $p_k < \theta$, $p_{k} - p_{k^*} < 0$ and ${p}_{k} + {p}_{{k^*}} < 2\theta$ (using the fact that $p_k \leq \min(p_{k^*},2\theta - p_{k^*})$). Then from Hoeffding's inequality,  
\begin{align*}
\bP(\hat{d}^r_{k^*} > \hat{d}^r_k) & \leq 
	\bP(\hat{p}_{k,r} > \theta)
	+ \bP(\hat{p}_{{k},r} - \hat{p}_{k^*,r} > 0)
	+ \bP(\hat{p}_{k^*,r} + p_{k,r} > 2\theta)
\\ & \leq 
	\exp\left( -2t_r (\theta - p_k)^2 \right\}
	+ \exp\left\{ - \frac{t_r}{2} (p_{k^*} - p_k)^2 \right\}
\\ & \hspace{0.4cm}
	+ \exp\left\{ - \frac{t_r}{2} (2\theta - p_{k^*} - p_k)^2 \right)
\\ & \leq 
	3 \exp\left( -\frac{t_r}{2} \min\left\{(\theta - p_k)^2,
		(p_{k^*} - p_k)^2,
		(2\theta - p_{k^*} -p_k)^2
	  \right\} \right)
\\ & = 
	3 \exp\left( -\frac{t_r}{2} \min\left\{
		(p_{k^*} - p_{k})^2,
		((2\theta - p_{k^*}) - p_k)^2
	  \right\} \right)
\end{align*}
which proves Equation~\ref{toproofin4case} as $\Delta_k^2 =\min\left\{
		(p_{k^*} - p_{k})^2,
		((2\theta - p_{k^*}) - p_k)^2
	  \right\}$. 
\end{proof}

Building on Lemma~\ref{lemma-inversion}, the next step is to control the probability that the MTD is eliminated in phase $r$. The proof bears strong similarities with that of Lemma~4.3 in \cite{Karnin13}. It is given below for the sake of completeness. 

%
%
\begin{lemma}
\label{lemma-best-survives}
The probability that the MTD is eliminated at the end of phase $r$ is at most
\begin{align*}
9 \exp\left(
	- \frac{n}{8 \log_2 K} \cdot \frac{\Delta^2_{k_r}}{k_r}
\right)	
\end{align*}
where $k_r = K/2^{r + 2}$.
\end{lemma}
%
%

%
%
The end of the proof of Theorem~\ref{thm:SH} is identical to than of Theorem~4.1 in \cite{Karnin13}, except that it uses our Lemma~\ref{lemma-best-survives}. We repeat the argument below with the appropriate modifications. 
Observe that if the algorithm recommends a wrong dose, the MTD must have been eliminated in one of t $\log_2(K)$ phases. Using Lemma~\ref{lemma-best-survives} and a union bound yields the upper bound
\begin{align*}
\bP\left(\hat{k}_n \neq k^*\right) &\leq 9 \sum_{r=1}^{\log_2 K} \exp \left(
	- \frac{n}{8 \log_2 K} \cdot \frac{\Delta^2_{k_r}}{k_r}
	\right)	
\\
	&\le 9 \log_2 K \cdot \exp \left(
		- \frac{n}{8 \log_2 K} \cdot \frac{1}{\max_k k \Delta^{-2}_k}
	\right)
\\
	&\le 9 \log_2 K \cdot \exp \left(
		- \frac{n}{8 H_2(\bm p) \log_2 K}
	\right),
\end{align*}
which concludes the proof.

\paragraph{Proof of Lemma~\ref{lemma-best-survives}}
Define $S_r'$ as the set of arms in $S_r$, excluding the
$\frac{1}{4}|S_r| = K/2^{r+2}$ arms with means closest to $\theta$.
If the MTD $k^*$ is eliminated in round $r$,
it must be the case that at least half the arms of $S_r$
(i.e., $\frac{1}{2}|S_r| = K/2^{r+1}$ arms)
have their empirical average closer to $\theta$ than its empirical
average.
In particular, the empirical means of at least
$\frac{1}{3}|S_r'| = K/2^{r+2}$ of the arms in $S_r'$ must be closer to
$\theta$ than that of the $k^*$ at the end of round $r$.
Letting $N_r$ denote the number of arms in $S_r'$ whose empirical average
is closer to $\theta$ than that of the optimal arm, we have by
Lemma~\ref{lemma-inversion}:
\begin{align*}
	\E[N_r] &= \sum_{k \in S_r'}
		\P(\hat d^r_k < \hat d^r_{k^*})
\\
	&\le \sum_{k \in S_r'} 3 \exp\left( -\frac{t_r}{2} \Delta^2_k \right)
\\
	&\le 3 \sum_{k \in S_r'} \exp\left( -\frac{1}{2} \Delta^2_k
		\cdot \frac{n}{|S_r| \log_2 K} \right)
\\
	&\le 3 |S_r'| \max_{k \in S_r'} \exp\left( -\frac{1}{2} \Delta^2_k
		\cdot \frac{2^r n}{K \log_2 K} \right)
\\
	&\le 3 |S_r'| \exp\left( -\frac{n}{8 \log_2 K}
		\cdot \frac{\Delta^2_{k_r}}{k_r} \right)
\end{align*}
Where the last inequality follows from the fact that there are at least
$k_r - 1$ arms that are not in $S_r'$ with average reward closer to $\theta$
than that of any arm in $S_r'$.
We now apply Markov's inequality to obtain
\begin{align*}
\P\left(N_r > \frac{1}{3}|S_r'|\right) &\le 3 \E[N_r] / |S_r'|
\\
	&\le 9 \exp \left(
		- \frac{n}{8 \log_2 K}
		\cdot \frac{\Delta^2_{k_r}}{k_r}
	\right),
\end{align*}
and the lemma follows.

\label{lastpage}

\end{document}